\documentclass{article}

\usepackage{arxiv}
\usepackage[utf8]{inputenc} 
\usepackage[T1]{fontenc}    
\usepackage{hyperref}       
\usepackage{url}            
\usepackage{booktabs}       
\usepackage{amsfonts}       
\usepackage{nicefrac}       
\usepackage{microtype}      
\usepackage{lipsum}		
\usepackage{graphicx}
\usepackage{natbib}
\usepackage{doi}

\usepackage{amsmath}
\usepackage{amssymb}
\usepackage{amsthm}
\usepackage{amsfonts}
\usepackage{pdfpages}
\usepackage{adjustbox}
\usepackage{booktabs}
\usepackage{algorithm}
\usepackage{algorithmic}

\newcommand{\myfls}{SAPE-FL }
\newcommand{\myfl}{SAPE-FL}

\newtheorem{theorem}{Theorem}
\newtheorem{assumption}{Assumption}
\newtheorem{lemma}{Lemma}

\usepackage{amsmath,amsfonts,bm}

\def\eqref#1{equation~\ref{#1}}

\def\1{\bm{1}}

\DeclareMathAlphabet{\mathsfit}{\encodingdefault}{\sfdefault}{m}{sl}
\SetMathAlphabet{\mathsfit}{bold}{\encodingdefault}{\sfdefault}{bx}{n}

\DeclareMathOperator*{\argmin}{arg\,min}

\usepackage{array}

\title{Similarity-Aware Personalized Federated Learning in Heterogeneous Environments}

\author{
	{Arun Kumar A V\thanks{Mail Correspondence: \texttt{a.anjanapuravenkatesh@deakin.edu.au}}$\:\:$, Sunil Gupta,  Dang Ngyuen, Bao Duong, Dat Phan Trong}  \\
	Applied Artificial Intelligence Initiative (A$^2$I$^2$)\\
	 Deakin University, Australia \\
}

\date{}

\renewcommand{\headeright}{}
\renewcommand{\undertitle}{}
\renewcommand{\shorttitle}{SAPEFL: Similarity-Aware Personalized Federated Learning}

\hypersetup{
pdftitle={SAPEFL-Similarity-Aware Personalized Federated Learning},
pdfsubject={SAPEFL},
pdfauthor={Arun Kumar A V},
pdfkeywords={Machine Learning, Heterogenous Environments, Federated Learning},
}

\begin{document}
\maketitle

\begin{abstract}
	 Federated Learning (FL) allows decentralized clients to train models collaboratively while preserving data privacy. However, distribution mismatch  across clients often leads to poor global generalization and degraded local client-level performance. In such scenarios, some of the clients with their local models trained solely on local data may perform better than the globally learnt model, thus nullifying the benefits of collaborative federated learning. To address this, we propose SAPE-FL (Similarity-Aware Personalized Federated Learning), a novel personalization framework that anchors each client's model to both the global model and a similarity-weighted peer averaged model. By incorporating dynamic, client-specific regularization based on both model similarity and output similarity, SAPE-FL adaptively balances global knowledge transfer and peer collaboration while filtering out dissimilar clients. This dual anchoring mitigates negative transfer and enhances robustness in heterogeneous settings. We theoretically analyze our algorithm establishing its convergence guarantees and empirically show that SAPE-FL outperforms state-of-the-art methods under high statistical heterogeneity and low client data regimes.
\end{abstract}


\section{Introduction\label{sec:introduction}}
With the rapid proliferation and growing computational power of edge devices such as smartphones, wearable sensors, and IoT devices, there has been a shift towards training machine learning models directly on decentralized devices rather than centralizing data in one location. Federated Learning (FL) \cite{fedavg_mcmahan2017communication} enables multiple clients to collaboratively train a global model without sharing raw data, thereby preserving privacy. The decentralized and privacy-preserving nature of FL is appealing across various domains 
including healthcare \cite{fed_covid}, 
finance \cite{fedfinance1_long2020} 
and edge computing \cite{edgecomp1_abreha2022}. 
Compared to traditional distributed learning approaches (Alternating Direction Method of Multipliers (ADMM) in \cite{admm_boyd}), FL methods have shown superior scalability and communication efficiency.

While traditional FL methodologies have demonstrated significant promise, real-world deployments often face substantial challenges due to heterogeneity \cite{fedchallenges_li2020}, including statistical (diverse data distributions across clients), model (variations in model training methods), and communication heterogeneity (inconsistent network conditions and resource constraints). Among these, statistical heterogeneity (distribution mismatch, henceforth referred to as non-IID) across clients poses a significant challenge, often degrading global model performance. Standard FL methods such as FedAvg \cite{fedavg_mcmahan2017communication}, originally designed for IID data, fail to generalize effectively across diverse client populations \cite{zhao2018federated}, 
motivating the development of personalization schemes that better align models with local client data.

To address data heterogeneity and personalization, several works have reformulated the optimization objectives of traditional FL methods. FedProx \cite{fedprox_li2020federated} adds a proximal term to FedAvg to limit the divergence of local updates from the global model. However, as noted in \cite{jiang2019improving}, optimizing solely for global accuracy can hinder effective personalization at the client level. To overcome this, personalized federated learning (PFL) \cite{pfl_sim2019,ditto_fl} has emerged as a paradigm that balances global collaboration with client-specific adaptation. Rather than relying on a single global model, PFL allows each client to tailor its model to local data and preferences \cite{survey_kulkarni2020}, typically following a two-stage strategy: global training followed by local fine-tuning.
However, existing PFL methods employ static regularization weights \cite{perfedavg_fallah2020personalized,ditto_fl}  or fixed similarity-based aggregation \cite{fedacs_chen2024personalized}, which \textit{fail to account for dynamic variations in client drift, data distribution, and model reliability over training rounds. This rigidity can lead to sub-optimal personalization or even negative transfer, particularly in settings with highly non-IID data.}

Motivated by the limitations of static personalization and uniform aggregation strategies, we propose \textbf{S}imilarity-\textbf{A}ware \textbf{PE}rsonalized \textbf{F}ederated \textbf{L}earning (\textbf{\myfl}), a novel framework that jointly optimizes global collaboration and client-specific adaptation. \myfls introduces a two-level anchoring mechanism, where each client aligns its model with both the global model and a similarity-weighted average of peer models. This dual anchoring is guided by client-specific adaptive regularization, computed using model output and weight-based similarity. Our measure combines both model output and model weight-based similarities, making it robust to structural and functional differences across clients. While weight-based similarity captures structural alignment, output similarity identifies functionally similar models despite parameter differences. To construct such reliable anchors, each client filters out dissimilar peer models using a similarity threshold, ensuring that only relevant peer information influences personalization. On the server side, \myfls performs similarity-weighted aggregation of client model updates, where each client's contribution is scaled by its alignment with peer consensus. This suppresses noisy or misaligned updates and enhances quality of the global model. 

We evaluate \myfls against state-of-the-art methods across diverse tasks, including synthetic data classification, human activity recognition and image classification under non-IID or statistical heterogeneity. Our empirical results demonstrate that \myfls framework consistently outperforms baselines, especially in settings with high statistical heterogeneity and limited client data availability. These findings highlight that our novel similarity-based personalization and aggregation strategy effectively personalizes federated learning to the unique characteristics of each client.

The key contributions of this paper are mentioned below.
\begin{itemize}
	
	\item We propose \myfl, a novel similarity-aware personalized FL framework that anchors client models to both global and peer-averaged models, and uses an aggregation scheme that scale client updates based on their coherence with peers and prevent negative transfer.
	
	\item We design an adaptive personalization scheme that computes client-specific regularization coefficients based on model output and weight similarity, enabling selective alignment with the global model and peer models.
	
	\item We provide theoretical insights into the convergence behavior and performance guarantees of our method. 
	
	\item We empirically validate \myfl's effectiveness on diverse datasets under high statistical heterogeneity.

\end{itemize}

Through these innovations, \myfls advances the state-of-the-art in personalized federated learning by jointly enabling global knowledge sharing and client-specific adaptation via similarity-aware optimization and aggregation.


\section{Related Work\label{sec:background}}
Recently, Federated Learning (FL) has gained traction in the research community resulting in a large and growing body of literature \cite{fedsurvey1_nevrataki}.
Federated learning, first introduced in \cite{fedavg_mcmahan2017communication}, enables decentralized model training while preserving data privacy. The FedAvg algorithm proposed allows clients to perform several rounds of local stochastic gradient descent (SGD) \cite{sgd_ruder2016overview} before synchronizing with server, reducing communication overhead compared to centralized training. Since its inception, FL has received growing attention 
leading to advancements in communication efficiency \cite{konevcny2016federated}, 
privacy guarantees \cite{agarwal2018cpsgd}, cryptographic protections \cite{bonawitz2017practical}, and optimization under resource constraints \cite{wang2019adaptive}. 
 
Despite strong theory and empirical performance \cite{li2019convergence}, FedAvg struggles under statistical heterogeneity. \cite{zhao2018federated} showed that non-IID data can degrade performance due to conflicting gradient updates. To mitigate these, adaptive FL methods have been proposed. \cite{wang2019adaptive} adjusted local update frequencies based on client resources, particularly suited for edge environments. LoAdaBoost \cite{loadaboosthuang2020} selectively continues training for under performing clients, reducing overhead and improving efficiency. Gradient-based clustering methods, such as CFL \cite{cfl_sattler2020clustered}, 
form clusters based on update similarity to mitigate heterogeneity.

\subsection{Adaptive Client Selection Methods} Optimizing client selection is crucial in FL to reduce communication and computation costs. \cite{chen2022optimal} proposed algorithms for efficient client participation, though without fully addressing heterogeneity. FedNorm \cite{fednorm_zhao2022} prioritizes clients with high informational value, improving communication efficiency. FedMarl \cite{fedmarl_zhang2022}, inspired by multi-agent reinforcement learning, dynamically selects clients based on runtime performance. CSFedAvg \cite{csfedavg_zhang2021client} improves convergence by favoring clients with less non-IID data. RIPFL \cite{qin2023reliable}, leveraging Dempster–Shafer theory \cite{sensoy2018evidential}, quantifies client reliability for robust selection.

\textbf{Multi-task and Meta-learning in Federated Learning} Multi-task and meta-learning approaches have gained traction in FL due to their ability to handle heterogeneous data distributions. MOCHA \cite{mocha_smith2017} jointly learns client-specific models and a task-similarity matrix to capture inter-client relationships. VIRTUAL \cite{virtual_corinzia2019variational} models inter-task variability and dependency through a variational multi-task framework. \cite{jiang2019improving} established connections between FedAvg and model-agnostic meta-learning (MAML) \cite{finn2017model}, which inspired methods like Per-FedAvg \cite{perfedavg_fallah2020personalized} that integrate meta-learning objectives into FL training. Other works include ARUBA \cite{aruba_khodak2019adaptive}, which improves gradient-based meta-learning, and \cite{peterson2019private}, which leverages mixture-of-experts (MoE) \cite{moe1_nowlan1990evaluation} 
for personalized model interpolation.

\subsection{Personalized Federated Learning Approaches} Personalization has become a core focus in FL due to the challenges posed by non-IID client data. Approaches range from local fine-tuning \cite{mansour2020three} 
to advanced decomposition-based methods \cite{fedper_arivazhagan2019federated,yin2025personalized} 
and clustering-based solutions \cite{ghosh2020efficient}.
The recently proposed LCFed \cite{zhang2025lcfed} introduces an efficient clustering framework that improves performance and communication efficiency in heterogeneous environments by forming compact client clusters using cosine-based distance metrics. Hypernetwork-based methods like PeFLL \cite{pefll_scott2023} generate personalized models in a single forward pass using learnt embeddings, while FedRod \cite{fedrod_chen2021bridging}, Scaffold \cite{scaffold_karimireddy2020scaffold}, and MOON \cite{moon_li2021model} impose regularization or similarity constraints to improve personalization during local updates. FedPAC \cite{fedpac_xu2023} aligns local and global feature embeddings, though it is most effective under label distribution shift. Model interpolation remains a prominent technique: FedProx \cite{fedprox_li2020federated} adds a proximal term to address client drift; Ditto FL \cite{ditto_fl} explicitly separates global and personalized models with fixed regularization. FedACS \cite{fedacs_chen2024personalized} incorporates attention-weighted peer model aggregation, while FedPer \cite{fedper_arivazhagan2019federated} separates personalized (top) layers and globally shared (base) layers in a neural network architecture.

Despite advancements, existing methods struggle in handling client heterogeneity and personalization \cite{yang2024fedas}. Many rely on static regularization, inefficient aggregation or clustering, which fail under dynamic data shifts. \myfls addresses this by anchoring client models to both global and peer-averaged models, using a hybrid similarity measure based on model weights and outputs. This enables robust peer matching and aggregation, ensuring effective personalization.

\section{Mathematical Preliminaries\label{sec:math_background}}
Traditional FL methods \cite{fedavg_mcmahan2017communication,fedavg_ft_collins2022} aim to collaboratively train a global model by aggregating local updates from multiple clients, without requiring each client's sensitive data to be pooled at a central location. Consider a FL setting with $C$ distributed clients, each holding data $\mathcal{D}_c$ drawn from a local distribution given as $\mathcal{D}_c = \{(\mathbf{x}_j^c, y_j^c)\}_{j=1}^{n_c}$, where $\mathbf{x}_j^c \in \mathbb{R}^{\text{dim}}$ are input features  and $y_j^c \in \{1, \dots, \mathcal{C}\}$ are class labels. The server maintains a global model $\theta_G$ which is iteratively updated based on client contributions. The objective in classical FL \cite{fedavg_mcmahan2017communication} is to learn a single global model $\theta_G$ that minimizes the weighted sum of local empirical risks:
\begin{equation} \label{eq:global_objective}
	\min_{\theta_G} \mathcal{L}(\theta_G) = \sum_c p_c \mathcal{L}_c(\theta_G)
\end{equation}
where $\mathcal{L}_c(\theta_G)$ denotes the local loss at client $c$, and $p_c$ is the participation probability assigned to each client, satisfying $\sum_c p_c = 1$ and $p_c \ge 0$. In the canonical example of FedAvg, $p_c$ is typically set proportional to the client's data volume: $p_c = \frac{n_c}{\sum_c n_c}$. This formulation assumes that all clients contribute equally to a single global objective, which may not be appropriate in heterogeneous (non-IID) settings.

\subsection{Statistical heterogeneity} Statistical heterogeneity in FL can be perceived in various forms. For instance, in a handwritten number recognition task, clients may have: (i) differing label distributions $p_i(y) \neq p_j(y) \; \forall i, j \in {C}$, or (ii) differing feature distributions (different writing styles - stroke/thickness)  $p_i(\mathbf{x}) \neq p_j(\mathbf{x})$ despite sharing the same label distribution $p_i(y) = p_j(y)$. Such heterogeneity can significantly degrade global model performance. To address this, personalized federated learning (PFL) methods aim to learn client-specific models $\theta_c$ that leverage global knowledge while adapting to local data.

The personalized optimization objective is defined as:
\begin{equation} \label{eq:pfl_formulation} 
	\Theta^\star = \theta_1^\star, \dots, \theta_C^\star = \argmin_{\theta_1, \dots, \theta_C \in \Theta} \; \sum_{c=1}^C p_c \mathcal{L}_c(\theta_c),
\end{equation}
where \( \theta_c^\star \) is the optimal model for client $c$, and $ \mathcal{L}_c(\theta_c)$ denotes the local empirical risk:
\begin{equation} \label{eq:localrisk}
	\mathcal{L}_c(\theta_c) \triangleq \frac{1}{|\mathcal{D}_c|} \sum_{(\mathbf{x}, y) \in \mathcal{D}_c} \ell(\theta_c(\mathbf{x}), y),
\end{equation}
where \( \ell(\cdot, \cdot) \) is a loss criterion (cross-entropy) and dataset $\mathcal{D}_c$ consists of input-output pairs $(\mathbf{x}, y)$, where $\mathbf{x} \in \mathbb{R}^{\text{dim}}$.

To better handle statistical heterogeneity, proximal-based PFL methods \cite{fedprox_li2020federated,ditto_fl} extend this objective by introducing regularization terms that constrain the drift of local models from a shared global reference.

\section{Framework\label{sec:framework}}
The key objective of this paper is to develop a federated learning framework to facilitate collaborative learning across clients with non-IID data distributions $\mathcal{D}_{c}$. As mentioned in the \textit{Mathematical Preliminaries} section, we consider a personalized federated learning of $C$ clients in a heterogeneous environment, i.e., the data distribution across clients $\mathcal{D}_{c}$ is non-IID. Precisely, we aim to learn the optimal personalized client models $\Theta^\star$ (Eq. (\ref{eq:pfl_formulation})).

To efficiently address the challenges posed by statistical heterogeneity in federated learning setting, we propose \myfls framework, a novel \textbf{S}imilarity-\textbf{A}ware \textbf{PE}rsonalized \textbf{F}ederated \textbf{L}earning (SAPE-FL) framework. \myfls is designed to adaptively regularize and balance local personalization and global collaboration by introducing (i) dual anchoring of each client model to both the global model and a peer-averaged model, and (ii) similarity-aware aggregation of client updates at the server. The mathematical formulation of the
\myfls framework is given as:
\begin{equation} \label{eq:myfl_model}
	\Theta^\star = \argmin_{\theta_1, \cdots, \theta_C} \sum_c p_c \mathcal{L}_c(\theta_c) + \mathcal{S}_c^P (\theta_c, \bar{\theta}_c) + \mathcal{S}_c^G (\theta_c, \theta_G)
\end{equation}
where $\mathcal{L}_c(\theta_c)$ is the local empirical risk at client $c$ having a personalized model $\theta_c$, and $\mathcal{S}^G_c(\cdot,\cdot)$, $\mathcal{S}_c^P(\cdot,\cdot)$ are the client specific similarity-aware regularization (penalty) coefficients that enforce alignment with the global and peer-averaged models at any given round $t$, respectively. We discuss in detail $\mathcal{S}^G_c$, $\mathcal{S}_c^P, \bar{\theta}_c$ in the next sections. The global model $\theta_G$ is updated iteratively by aggregating client updates in a similarity-weighted manner. 

We present the complete workflow of the proposed \myfls framework in Algorithm \ref{alg:server} and Algorithm \ref{alg:client}.

\subsection{Dual Anchoring: Similarity-Aware Regularization}

To effectively balance personalization and global collaboration, \myfls introduces a dual anchoring mechanism. Each client optimizes its local model by anchoring simultaneously to the global model $\theta_G$ and to a peer-averaged model $\bar{\theta}_c$, constructed using a subset of similar clients. This dual anchoring strategy helps clients incorporate useful knowledge from both centralized and localized sources while avoiding influence from dissimilar or noisy models. The personalized client objective is defined as:
\begin{equation}
	\label{eq:personalized_client_loss}
	\min_{\theta_c} \mathcal{L}_c(\theta_c) + \mathcal{S}_c^P(\theta_c, \bar{\theta}_c) + \mathcal{S}_c^G(\theta_c, \theta_G)
\end{equation}
The similarity-aware regularization coefficients $\mathcal{S}_c^G$, $\mathcal{S}_c^P$ are defined as a convex combination of (i) model output-based and (ii) weight-based similarity. We use this hybrid similarity metric to enhance the robustness of \myfls to behavioral differences and structural variations among clients. We define the coefficients $\mathcal{S}_c^G$, $\mathcal{S}_c^P$ as: 
\begin{equation}
	\begin{aligned}
	\label{eq:def_reg_strength}
	\mathcal{S}_c^P \triangleq \delta \cdot \vartheta(f(\theta_c, \mathcal{D}_c), f(\bar{\theta}_c, \mathcal{D}_c)) + (1-\delta) \cdot \vartheta (\theta_c, \bar{\theta}_c)\\
	\mathcal{S}_c^G \triangleq \delta \cdot \vartheta(f(\theta_c, \mathcal{D}_c), f(\theta_G, \mathcal{D}_c)) + (1-\delta) \cdot \vartheta (\theta_c, \theta_G) 
	\end{aligned}
\end{equation}
where $\delta \in [0,1]$ is a tunable influence factor that balances emphasis between output similarity (model predictions) and structural similarity (model weights) and $\vartheta(\cdot, \cdot)\rightarrow[0,1]$ is a positive function defined for any two inputs $\mathcal{X}, \mathcal{X}'$ as:
\begin{equation}
	\vartheta(\mathcal{X}, \mathcal{X}') = 1-\max(0, \cos(\mathcal{X}, \mathcal{X}'))
\end{equation}

In the definitions of $\mathcal{S}_c^G$ and $\mathcal{S}_c^P$, the first term penalizes discrepancies in \textit{model outputs} by measuring the alignment between the predictive behaviors $f(\theta,\mathcal{D}_c)$ of the two models evaluated on the same client dataset $\mathcal{D}_c$. The second term imposes a penalty based on \textit{structural alignment}, capturing directional consistency between the model parameters. Unlike conventional parameter-distance regularization schemes \cite{fedprox_li2020federated,zhang2025lcfed}, which explicitly enforce Euclidean proximity between model parameters, the proposed similarity-aware penalties do not induce direct parameter contraction. Instead, they promote functional alignment and directional consistency with the global and peer-averaged anchors, allowing structurally distinct yet behaviorally consistent personalized models to coexist.

Further, the similarity-based regularization coefficients adaptively regulate the trade-off between global collaboration and local personalization. Larger regularization coefficients strengthen the alignment with corresponding anchor model (either global or peer-averaged), encouraging knowledge sharing among aligned clients, while smaller values relax the regularization effect and enable stronger local adaptation, thereby mitigating the risk of negative transfer.

\subsection{Peer Model Filtering and Averaging}

To construct the peer-averaged model $\bar{\theta}_c^t$, each client evaluates the relevance of peer client models by computing similarity scores and filters out dissimilar peers based on a predefined threshold $S_{\text{min}}$. The parameter $S_{\text{min}}$ is carefully selected to balance inclusiveness and robustness. It ensures only peers that are sufficiently aligned contribute to the personalized anchor. The peer-averaged model is given by:
\begin{equation} \label{eq:avg_peer_model}
	\bar{\theta}_c = \frac{1}{\sum_j S_{cj} \cdot \mathbb{I}[S_{cj} > S_{\text{min}}]} \sum_j S_{cj} \cdot \theta_j \cdot \mathbb{I}[S_{cj} > S_{\text{min}}]
\end{equation}
where $\mathbb{I}[\cdot]$ is the indicator function. Analogous to Eq. (\ref{eq:def_reg_strength}), the similarity score $S_{cj}$ between client $c$ and peer $j$ is given as: 
\begin{equation} \label{eq:combined_sim}
	S_{cj} = \delta \cdot \tilde{S}_{cj} + (1 - \delta) \cdot \hat{S}_{cj}
\end{equation}
where $\tilde{S}_{cj}$ is the \textit{output similarity} term measuring the alignment between predictive behaviors of the two models: 
\begin{equation}
	\label{eq:output_sim}
	\tilde{S}_{cj} = 1-\vartheta \big(f(\theta_c, \mathcal{D}_c), f(\theta_j, \mathcal{D}_c)\big)
\end{equation}
where \( f(\theta, \mathcal{D}_c) \) denotes the outputs of model $\theta$ on dataset \( \mathcal{D}_c \) and $\hat{S}_{cj}$ is the \textit{weight similarity} measuring the structural alignment between the model parameters defined as:
\begin{equation}
	\label{eq:model_sim}
	\hat{S}_{cj} = 1-\vartheta(\theta_c, \theta_j)
\end{equation}

In Eq. (\ref{eq:combined_sim}), $\tilde{S}_{cj}$ term captures whether the peer produces similar predictions on the client’s data, whereas $\hat{S}_{cj}$ provides an additional information based on the proximity of model weights, which is useful when output similarity alone is noisy or insufficient (especially during early rounds). This scheme enables each client to construct a robust and personalized peer anchor by leveraging only those peers that are behaviorally and structurally aligned, thereby improving personalization and robustness to noisy or adversarial updates. 


\subsection{Similarity-Weighted Aggregation at Server}

Following client-level optimization of personalized models, each client computes its model update $\Delta_c^t$ as 
\[\Delta_c^t= \tilde{\theta}_c^{t} - {\theta}_G^t\] 
where $\tilde{{\theta}}_c^t$ is the client model initialized with $\theta_G$ and updated locally for $E$ epochs. The model update $\Delta_c^t$ is transmitted back to the server along with the aggregate peer similarity $s_c = \sum_j S_{cj}$ and the optimal client model $\theta_c^{t+1}$. The global model $\theta_G^t$ is updated via a similarity-weighted aggregation:
\begin{equation}
	\label{eq:global_update}
	\theta_G^{t+1} = \theta_G^t + \frac{1}{\sum_c s_c^t} \sum_c s_c^t \cdot \Delta_c^t
\end{equation}
This aggregation gives more weight to clients whose models align with their peers, thereby encouraging updates that represent consensus learning behavior. Clients with lower scores are down-weighted, thus suppressing noisy updates.

This end-to-end scheme of dual anchoring through similarity-aware regularization, selective peer model averaging, and similarity-weighted aggregation makes \myfls robust against data heterogeneity and misaligned signals.

\subsection{Algorithmic Workflow of \myfls Framework}

We now present the complete workflow of our method. The server-side procedure at round $t$ is in Algorithm \ref{alg:server} and the client-side procedure for personalized federated learning at client $c$ is in Algo. \ref{alg:client}. In each round, a series of coordinated steps between the server and clients are performed to progressively refine both global and personalized client models. 

The server begins each round by sampling a subset of clients according to the participation ratio $r$ and broadcasting the current global model $\theta_G^t$ along with peer models, to the selected clients. Each participating client initializes its local model $\tilde{\theta}_c^t \gets \theta_G^t$, and updates it locally for $E$ epochs using its private dataset $\mathcal{D}_c$. Following this, each client computes the combined similarity score $S_{cj}$ with its peers by computing output similarity $\tilde{S}_{cj}$ and weight similarity $\hat{S}_{cj}^t$. Models with similarity scores above threshold $S_{\min}$ are then used to compute a peer-averaged model $\bar{\theta}_c^t$ (Eq. (\ref{eq:avg_peer_model})).

With both the global and peer anchors, each client solves a personalized optimization problem that balances empirical risk minimization with regularization toward the global and the peer-averaged model. The regularization coefficients $\mathcal{S}_c^P$ and $\mathcal{S}_c^G$ are computed based on the client's alignment with the respective anchors. After personalized training, each client retains its personalized model and computes the model update $\Delta^t_c$ which gets transmitted back to the server alongside the updated model ${\theta}^{t+1}_c$ and similarity scores $s_c$.

\begin{algorithm}[t]
	\caption{Server-Side Procedure (Round $t$)} \label{alg:server}
	\begin{algorithmic}[1]
		\STATE \textbf{Input:} Client participation ratio $r$
		
		\STATE Randomly sample $C$ clients according to ratio $r$
		
		\FOR {each client $c \in C$ \textbf{in parallel}}
		\STATE Send global model $\theta_G^t$ and peer models $\theta_j^{t}$ (Algo. \ref{alg:client})
		
		\STATE Receive model $\theta_c^{t+1}$, update $\Delta_c^t$, similarity score $s_c^t$ 
		\ENDFOR
		\STATE Update global model using similarity scores: $
		\theta_G^{t+1} = \theta_G^t + \frac{1}{\sum_{c} s_c^t} \sum_{c} s^t_c \cdot \Delta_c^t$
		\RETURN $\theta_G^{t+1}$ as the updated global model
	\end{algorithmic}
\end{algorithm}

\begin{algorithm}[t]
	\caption{Client-Side Procedure (Client $c$)} \label{alg:client}
	\begin{algorithmic}[1]
		\STATE \textbf{Input:} Global model $\theta_G^t$, Peer models $\{\theta_j^{t}\}$, Data $\mathcal{D}_c$
		\STATE \textbf{Parameters:} Similarity threshold $S_{\text{min}}$, Influence factor $\delta\!\!$
		
		\STATE Initialize local model: $\tilde{\theta}_c^t \gets \theta_G^t$
		\STATE Update $\tilde{\theta}_c^t$ on $\mathcal{D}_c$ for $E$ epochs
		\FOR {each peer model $\theta_j^{t}$}
		\STATE Compute output similarity $\tilde{S}_{cj}^t$ as per Eq. (\ref{eq:output_sim})
		\STATE Compute weight similarity: $\hat{S}_{cj}^t$ as per Eq. (\ref{eq:model_sim}) 
		\STATE Combined similarity score: $S_{cj}^t = \delta \tilde{S}_{cj}^t + (1 - \delta) \hat{S}_{cj}^t$
		\ENDFOR
		
		\STATE Build peer-averaged model $\bar{\theta}_c^t$ as per Eq. (\ref{eq:avg_peer_model})
		
		\STATE Compute $\mathcal{S}_c^{G}, \mathcal{S}_c^{P}$ as per Eq. (\ref{eq:def_reg_strength}) 
		
		\STATE Find the optimal model $\theta_c^{t+1}$ that minimizes the loss: $		\argmin_{\theta_c}\mathcal{L}(\theta_c) + \mathcal{S}^{P}_c(\theta_c, \bar{\theta}_c^t)+ \mathcal{S}_c^{G}(\theta_c, \theta_G^t)$
		\STATE Compute model update $\Delta_c^t =  \tilde{\theta}_c^t - \theta_G^t$		
		\STATE Compute aggregate similarity score: $s_c^t = \sum_j S_{cj}^t$
		
		\RETURN model $\theta_c^{t+1}$,  update $\Delta_c^t$, scores $s_c$ to server
	\end{algorithmic}
\end{algorithm}

At the server, these similarity-weighted updates are aggregated to refine the existing global model $\theta_G^{t}$. This similarity-aware aggregation ensures that clients whose updates are well-aligned with their peers (higher $s_c$) contribute strongly to the global update, while noisy or dissimilar updates are naturally down-weighted. The refined global model $\theta_G^{t+1}$ is then shared in the next round, continuing the cycle of personalized optimization and collaborative learning.

\section{Convergence Analysis of \myfls Method} \label{sec:convergence}
In this section, we analyze the theoretical aspects of our proposed \myfls algorithm. Our goal is to understand the convergence behavior of the global model and personalized models under the proposed similarity-weighted personalization and aggregation scheme. We aim to formally prove a per-round improvement bound for the global model objective using similarity-weighted model updates from participating clients. For analysis, we consider the setting where each client $c$ performs local training initialized at the global model $\theta_G$, and constructs its model update $\Delta_c$ based on local and peer-regularized learning. The server then refines the global model using the client model updates received. We first state some key assumptions required for the convergence analysis of our proposed \myfls framework:

\begin{assumption}[\textbf{Smoothness}]\label{assum:smooth} Each client objective $\mathcal{L}_c(\theta)$ is $L$-smooth: for all $\theta,\theta'$, $\|\nabla \mathcal{L}_c(\theta_c) - \nabla \mathcal{L}_c(\theta_c^\prime)\| \le L\|\theta_c-\theta_c^\prime\|$. 
\end{assumption}

\begin{assumption}[\textbf{Strong $\mu$-Convexity}] \label{assum:convexity} Each client's local objective function $\mathcal{L}_c$ is strongly $\mu$-convex, formally: \[ \mathcal{L}_c(\theta_c) \geq \mathcal{L}_c(\theta_c^\prime) + \langle \nabla \mathcal{L}_c(\theta_c^\prime), \theta_c - \theta_c^\prime  \rangle + \frac{\mu}{2} \|\theta_c-\theta_c^\prime\|^2 \:\:\:\:\;\forall \theta_c, \theta_c^\prime\]
\end{assumption}

\begin{assumption}[\textbf{Gradient Bounded Variance}]\label{assum:grad_bound} The gradients on each client  c (mini-batch sampling) have bounded variance. For some $\sigma^2\ge0$, $\mathbb{E}[\|\nabla \mathcal{L}_c(\theta; \xi) - \nabla \mathcal{L}_c(\theta)\|^2] \le \sigma^2$ $\forall c,\theta$, where $\xi$ indexes a random data sample.
\end{assumption}

Assumptions \ref{assum:smooth}, \ref{assum:convexity}, and \ref{assum:grad_bound} are standard in the FL theory found in the literature \cite{pfedme_t2020personalized,li2019convergence,fedprox_li2020federated}. Assumption \ref{assum:smooth} is used to bound the improvement in loss after each update. Assumption \ref{assum:convexity}, commonly adopted in FL theory, is required to establish linear convergence rates under appropriate conditions. Assumption \ref{assum:grad_bound} bounds the gradients and ensures that the noise introduced in local updates does not cause divergence. Under the assumptions stated above, we prove the descent condition of the expected global objective using the following theorem. 

\begin{theorem}[\textbf{One-Round Global Model Improvement}]
	\label{thm:descent}
	Let $\mathcal{L}_c(\theta_c)$ be the client's $L$-Lipschitz smooth objective with bounds on its gradients. Let the similarity scores $s^t_c$ also be bounded and non-negative. After each client trains locally, the similarity-aware aggregation satisfies the inequality: 
	\[ \sum_c \frac{{s}_c^t}{\sum_j {s}_j^t} \nabla \mathcal{L}(\theta_G^t)^\top \Delta_c^t \leq -\rho \|\nabla \mathcal{L}(\theta_G^t)\|^2 \text{ for some } \rho>0\]
	Then the global model satisfies the descent condition:
	\[ \mathbb{E}[\mathcal{L}(\theta_G^{t+1})] \leq \mathcal{L}(\theta_G^t) - \rho \|\nabla \mathcal{L}(\theta_G^t)\|^2 + \frac{L}{2} \mathbb{E}[\|\theta_G^{t+1} - \theta_G^t\|^2]. \]
\end{theorem}

The detailed proof of Theorem \ref{thm:descent} is provided in Appendix. The proof leverages the $L$-smoothness of the global objective $\mathcal{L}(\theta_G)$ and expands $\mathcal{L}(\theta_G^{t+1})$ using the standard Taylor expansion. The aggregated update scheme used in our proposed framework is shown to form a descent direction under the assumption that each client update $\Delta_c^t$ approximately aligns with the negative local gradient $-\eta\nabla \mathcal{L}(\theta_G^t)$ \cite{fedprox_li2020federated,ditto_fl}. The similarity scores act as attention weights, emphasizing aligned/useful updates and reducing the impact of noisy or misaligned clients. The derived inequality implies that the expected global loss decreases at each communication round, provided that client updates are sufficiently aligned and similarity weights are effective. 

Next, to understand the impact of peer-based collaboration in our personalized framework, we examine how aggregating peer models influences the statistical properties of each client’s personalized loss. Specifically, we consider the peer-averaged model $\bar{\theta}_c$. When peer models are aligned in behavior, their aggregation acts as a form of ensembling that retains unbiasedness and reduces variance in the resulting estimate. The following lemma formalizes this variance reduction, illustrating the statistical benefits of incorporating similar peer models into local personalization.


\begin{lemma}[\textbf{Peer-Average Variance Reduction}]
\label{lem:peer_avg_var}
Let $C_k$ be the set containing client $c$ and its similar peers $j$. Assume all peer models share the same mean, i.e.\ $\mathbb{E}[\theta_j^t]=\mathbb{E}[\theta_{j'}^t]$ for every $j,j'\in C_k$, and have covariance $\mathrm{Cov}(\theta_j^t)=\Sigma$. With the per-client estimation-error variance defined as $\upsilon=\mathrm{Tr}(\Sigma)$, the peer-averaged model $\bar{\theta}_c^t$ retains the same mean as the individual peers and enjoys reduced variance $\mathrm{Cov}(\bar{\theta}_c^t)=\frac{1}{|C_k|}\,\Sigma$. Consequently, the estimation-error variance of the peer average is  	$\bar{\upsilon}=\mathrm{Tr}(\mathrm{Cov}(\bar{\theta}_c^t))=\frac{1}{|C_k|}\,\upsilon$.
\end{lemma}

We provide the proof of Lemma \ref{lem:peer_avg_var} in Appendix. The proof leverages the independence and identical expectation of peer models under the cluster to show that the peer-averaged estimate $\bar{\theta}_c^t$ remains unbiased while reducing the variance by a factor of $\frac{1}{|C_k|}$. This is analogous to ensemble learning where combining multiple independent estimates improves robustness and accuracy. We show that the loss incurred by $\bar{\theta}_c^t$ is at most the average loss of individual peer models, further justifying its use as personalized anchor.

We study the convergence of the client objective that combines local loss minimization with regularization to align with the global and peer-averaged model. The similarity-weighted averaging and anchoring enables information sharing within client neighborhoods while preserving privacy and heterogeneity across the federation. The following theorem formalizes convergence of the client objective.

\begin{theorem}[\textbf{Convergence of Personalized Objective}] \label{thm:pers_obj}
	If each participating client solves the sub-problem given as:
	\[
	\min_{\theta_c}\mathcal{L}_c(\theta_c)\;+\;\mathcal{S}_c^P(\theta_c, \bar{\theta}_c^t)+\mathcal{S}_c^G(\theta_c, \theta_G^t),
	\]
	where $\bar{\theta}_c^t$ is computed via similarity scores $S\in [0,1]$ mentioned in Eq. (\ref{eq:combined_sim}) and filtering threshold $S_{\text{min}}$ to exclude dissimilar peers. Then the global objective (aggregated objective across all clients)
	\[
	\mathcal{L}(\theta_{\{1\cdots C\}})\!=\!\sum_c \mathcal{L}_c(\theta_c)+\sum_c\! \left({\mathcal{S}_c^P}(\theta_c, \bar{\theta}_c^t)\!+\!{\mathcal{S}_c^G}(\theta_c,\theta_G^t)\right)
	\]
	decreases monotonically and the method converges to a first-order stationary point. With appropriate step-size $\eta>0$, an $\epsilon$-approximate solution is reached in $\mathcal{O}\left(\frac{1}{|C_k|\epsilon^2}\right)$ communication rounds, where $|C_k|$ is the number of similar peers ($S_{cj}^t>S_{\text{min}}$) for a given client $c$.
\end{theorem}

Theorem \ref{thm:pers_obj} shows that \myfls converges to a stationary point with monotonic decrease in the joint objective.  As per Lemma \ref{lem:peer_avg_var}, the variance of the peer-averaged estimate is reduced by a factor of $\frac{1}{|C_k|}$. This tighter variance directly translates into an improved convergence rate of $\mathcal{O}\left({\frac{1}{|C_k|\epsilon^2}}\right)$, faster than the classical $\mathcal{O}\left({\frac{1}{\epsilon^2}}\right)$ bound \cite{fedamp_huang2021personalized,fedavg_ft_collins2022}. The global model stabilizes via similarity-aware aggregation, while the local models evolve toward cluster-aligned personalized optima. 

Additionally, we theoretically prove that our similarity-weighted global objective converges linearly in the strongly convex case and attains the standard first-order rate in the non-convex case. The associated theorems and their proofs are provided in the appendix due to space constraints.

\section{Experiments\label{sec:Experiments}}
In this section, we evaluate the empirical performance of our \myfls framework by comparing it with a set of relevant state-of-the-art federated learning and personalized FL baselines. We first provide the experimental settings and then discuss the empirical results of our framework. 


\subsection{Baselines}
Given the breadth of FL literature, we focus our comparisons on methods that are conceptually aligned with our \myfls approach. For completeness, we includes a local-only baseline where each client trains independently without collaboration. The baselines considered are: (i) \textbf{FedAvg} \cite{fedavg_mcmahan2017communication}, (ii) \textbf{PerFedAvg} \cite{perfedavg_fallah2020personalized}, (iii) \textbf{FedProx} \cite{fedprox_li2020federated}, (iv) \textbf{Ditto FL} \cite{ditto_fl}, (v) \textbf{FedACS} \cite{fedacs_chen2024personalized}, (vi) \textbf{LCFed} \cite{zhang2025lcfed}, (vii) FedAFK \cite{yin2025personalized}, (viii) FedAS \cite{yang2024fedas}.

\begin{table}[!h]
	\centering
	\caption{\textbf{Accuracy Scores ($\%$)}. Comparison of \myfls with baseline methods. Each cell signifies the mean accuracy (standard deviation) computed across all clients at the end of all communication rounds. Bold values indicate the best.}
	\label{tab:acc}
	
	\resizebox{\textwidth}{!}{
		\begin{tabular}{|c|c|c|c|c|c|c|c|c|}
			\toprule
			\textbf{Method} & \textbf{Synthetic} & \textbf{HAR} & \textbf{AccGyro} & \textbf{CIFAR-10} & \textbf{CIFAR-100} & \textbf{MNIST} & \textbf{FMNIST} & \textbf{EMNIST} \\
			\midrule
			
			\textbf{Local-only}
			& $80.64\pm0.24$
			& $73.25\pm0.61$
			& $75.27\pm0.92$
			& $77.45\pm0.86$
			& $41.66\pm0.35$
			& $68.45\pm0.11$
			& $82.40\pm0.55$
			& $63.26\pm0.23$ \\
			
			\textbf{FedAvg}
			& $83.19\pm0.25$
			& $76.61\pm0.33$
			& $78.15\pm0.84$
			& $79.23\pm0.16$
			& $42.31\pm0.68$
			& $72.24\pm0.23$
			& $84.32\pm0.15$
			& $67.22\pm0.16$ \\
			
			\textbf{PerFedAvg}
			& $87.66\pm0.23$
			& $77.07\pm0.87$
			& $80.12\pm0.18$
			& $80.35\pm0.36$
			& $43.81\pm0.06$
			& $79.24\pm0.68$
			& $88.67\pm0.27$
			& $69.28\pm0.11$ \\
			
			\textbf{FedProx}
			& $86.24\pm0.28$
			& $76.37\pm0.95$
			& $77.41\pm0.45$
			& $76.21\pm0.91$
			& $42.79\pm0.48$
			& $78.27\pm0.18$
			& $87.61\pm0.32$
			& $70.45\pm0.92$ \\
			
			\textbf{Ditto FL}
			& $86.81\pm0.97$
			& $76.59\pm0.39$
			& $79.07\pm0.57$
			& $79.97\pm0.44$
			& $43.29\pm0.42$
			& $75.66\pm0.08$
			& $89.05\pm0.52$
			& $71.26\pm0.82$ \\
			
			\textbf{FedACS}
			& $89.94\pm0.01$
			& $78.22\pm0.39$
			& $80.66\pm0.82$
			& $78.26\pm0.37$
			& $43.92\pm0.98$
			& $80.25\pm0.19$
			& $\mathbf{89.96\pm0.71}$
			& $70.80\pm0.60$ \\
			
			\textbf{LCFed}
			& $87.98\pm0.61$
			& $79.07\pm0.77$
			& $74.23\pm0.16$
			& $81.02\pm0.67$
			& $\mathbf{44.23\pm0.85}$
			& $80.97\pm0.12$
			& $88.65\pm0.84$
			& $71.07\pm0.63$ \\
			
			\textbf{FedAFK}
			& $85.81\pm0.60$
			& $78.80\pm0.82$
			& $80.01\pm0.71$
			& $77.19\pm0.40$
			& $44.12\pm0.24$
			& $78.71\pm0.66$
			& $85.70\pm0.61$
			& $66.90\pm0.56$ \\
			
			\textbf{FedAS}
			& $88.11\pm0.20$
			& $80.20\pm0.32$
			& $79.59\pm0.81$
			& $76.60\pm0.13$
			& $42.20\pm0.52$
			& $77.15\pm0.28$
			& $85.60\pm0.21$
			& $68.93\pm0.09$ \\

			\textbf{\myfl}
			& $\mathbf{90.83\pm0.49}$
			& $\mathbf{81.94\pm0.77}$
			& $\mathbf{82.69\pm0.55}$
			& $\mathbf{82.07\pm0.96}$
			& $44.01\pm0.63$
			& $\mathbf{81.93\pm0.61}$
			& $89.01\pm0.44$
			& $\mathbf{72.35\pm0.70}$ \\
			
			\bottomrule
		\end{tabular}
	}
\end{table}

\begin{table}[h]
	\centering
	\caption{\textbf{AUC Scores ($\%$)}. Comparison of \myfl with baseline methods. Each cell signifies the mean AUC (standard deviation) computed across all clients at the end of all communication rounds. Bold indicate the best}
	\label{tab:auc}
	
	\resizebox{\textwidth}{!}{
		\begin{tabular}{|c|c|c|c|c|c|c|c|c|}
			\toprule
			\textbf{Method} & \textbf{Synthetic} & \textbf{HAR} & \textbf{AccGyro} & \textbf{CIFAR-10} & \textbf{CIFAR-100} & \textbf{MNIST} & \textbf{FMNIST} & \textbf{EMNIST} \\
			\midrule
			
			\textbf{Local-only}
			& $82.21\pm0.19$
			& $74.22\pm0.79$
			& $75.60\pm0.82$
			& $78.40\pm0.31$
			& $56.25\pm0.67$
			& $72.45\pm0.28$
			& $79.39\pm0.38$
			& $78.07\pm0.29$ \\
			
			\textbf{FedAvg}
			& $84.51\pm0.63$
			& $77.16\pm0.22$
			& $78.24\pm0.67$
			& $80.22\pm0.49$
			& $59.64\pm0.52$
			& $74.43\pm0.25$
			& $81.07\pm0.52$
			& $80.67\pm0.44$ \\
			
			\textbf{PerFedAvg}
			& $86.24\pm0.37$
			& $78.29\pm0.07$
			& $87.64\pm0.08$
			& $82.16\pm0.57$
			& $65.67\pm0.88$
			& $82.60\pm0.17$
			& $83.39\pm0.87$
			& $81.29\pm0.97$ \\
			
			\textbf{FedProx}
			& $87.89\pm0.39$
			& $76.92\pm0.50$
			& $80.59\pm0.40$
			& $79.62\pm0.32$
			& $66.59\pm0.67$
			& $83.35\pm0.19$
			& $82.03\pm0.16$
			& $83.45\pm0.21$ \\
			
			\textbf{Ditto FL}
			& $86.16\pm0.67$
			& $76.81\pm0.16$
			& $84.22\pm0.43$
			& $82.43\pm0.27$
			& $66.29\pm0.80$
			& $80.31\pm0.49$
			& $87.89\pm0.12$
			& $84.48\pm0.26$ \\
			
			\textbf{FedACS}
			& $87.95\pm0.50$
			& $78.75\pm0.67$
			& $87.51\pm0.71$
			& $81.08\pm0.11$
			& $67.72\pm0.58$
			& $84.26\pm0.77$
			& $87.09\pm0.87$
			& $83.19\pm0.18$ \\
			
			\textbf{LCFed}
			& $87.02\pm0.61$
			& $\mathbf{79.97\pm0.13}$
			& $83.19\pm0.75$
			& $82.10\pm0.65$
			& $\mathbf{70.07\pm0.62}$
			& $84.49\pm0.33$
			& $88.02\pm0.09$
			& $83.69\pm0.76$ \\
			
			\textbf{FedAFK}
			& $82.20\pm0.17$
			& $72.64\pm0.80$
			& $81.23\pm0.52$
			& $78.47\pm0.29$
			& $59.92\pm0.77$
			& $81.35\pm0.70$
			& $85.19\pm0.32$
			& $84.44\pm0.17$ \\
			
			\textbf{FedAS}
			& $86.41\pm0.88$
			& $77.11\pm0.01$
			& $76.09\pm0.77$
			& $81.04\pm0.91$
			& $67.79\pm0.06$
			& $82.02\pm0.11$
			& $81.25\pm0.33$
			& $84.22\pm0.41$ \\
			
				\textbf{\myfl}
			& $\mathbf{88.64\pm0.04}$
			& $79.24\pm0.56$
			& $\mathbf{88.60\pm0.58}$
			& $\mathbf{84.04\pm0.14}$
			& $68.09\pm0.79$
			& $\mathbf{86.09\pm0.25}$
			& $\mathbf{88.69\pm0.45}$
			& $\mathbf{86.27\pm0.50}$ \\
			
			\bottomrule
		\end{tabular}
	}
\end{table}

\subsection{Data Partitions and Performance Metrics}

To simulate statistical heterogeneity across clients, we partition datasets using a Dirichlet distribution with concentration parameter $\kappa = 0.3$. This setup ensures non-IID data distributions typical of federated settings. We run \myfls algorithm for $200$ communication rounds and use $C=100$ clients for synthetic, HAR and image classification tasks. Clients data is split 80/20 into train and test sets. The client participation ratio ($r$) is set to $70\%$ per round. We set the influence factor $\delta$ to $0.5$ and the similarity threshold $S_{\text{min}} = 0.65$. We repeat all our experiments with $5$ random intializations. Model performance is reported using mean test accuracy ($\pm$standard deviation) and AUC ($\pm$standard deviation) across clients. While SAPE-FL framework introduces additional computation and communication overhead due to pairwise similarity estimation and peer maintenance, this cost can be effectively mitigated in large-scale deployments using matrix approximation techniques such as low-rank factorization, which substantially reduce per-client overhead. We refer to Appendix for the detailed experimental setup and discussion. 


We now discuss the empirical results of the \myfls framework compared with baselines. Table \ref{tab:acc} shows the mean test accuracy ($\pm$standard deviation) at the final communication round across all clients and Table \ref{tab:auc} reports the corresponding average AUC scores ($\pm$standard deviation). The proposed framework consistently outperformed baselines, achieving the highest mean test accuracy in most classification tasks, highlighting its ability to effectively personalize under heterogeneous conditions. While \myfls consistently performed well, methods such as Ditto FL, and LCFed exhibit competitive performance in few cases. Ditto FL's strong result may be attributed to its dual-model structure that isolates local and global learning paths. LCFed likely benefits from their clustering scheme that aligns well with structured character data. Nonetheless, \myfls demonstrates broad generalization and robust adaptation, outperforming baselines in HAR tasks where personalization is very crucial.

To further analyze and validate the robustness of the \myfls framework, we conducted additional experiments, the results of which are deferred to the \textit{Appendix} due to space constraints. In particular, we track the evolution of local client models across multiple rounds and demonstrate how local models (with initially poor generalization) benefit from guidance provided by the global and peer-averaged anchors, thereby mitigating negative knowledge transfer. We additionally perform an ablation study to understand the sensitivity of \myfls to hyperparameters such as influence factor $\delta$ and the similarity threshold $S_{\text{min}}$. The ablation study and additional experimental results are in the \textit{Appendix}.


%
%

\section{Conclusion}
In this work, we proposed Similarity-Aware PErsonalized Federated Learning (\myfl), a novel personalization framework designed to address statistical heterogeneity in federated learning. \myfls introduced a dual anchoring mechanism, where each client aligns its model with both the global and a similarity-weighted peer average model, guided by adaptive regularization based on both model parameters similarity and output similarity. Additionally, a similarity-aware aggregation strategy at the server prioritizes coherent client updates, enhancing robustness and personalization. Our results show that \myfls achieves superior performance under heterogeneous settings when compared to the state-of-the-art methods, thereby validating its importance.

\section{Acknowledgments}
This work was supported by the Wellcome Trust [303030/Z/23/Z].

\newpage

\bibliographystyle{unsrtnat}
\bibliography{BibFiles/FedLearn}

\newpage
\appendix

\section{Appendix}

We use lower case bold fonts $\mathbf{v}$ for vectors and $v_{i}$
for each element in $\mathbf{v}$. $\mathbf{v}^{\intercal}$ is the
transpose of $\mathbf{v}$. We use upper case bold fonts (and bold greek symbols) $\mathbf{M}$
for matrices and $M_{ij}$ for each element in $\mathbf{M}$. $\mid\cdot\mid$ is the cardinality of the set. ${\rm {\rm abs}(\cdot)}$
is the absolute value. $\mathbb{N}_{n}=\{1,2,\cdots,n\}$.
$\mathbb{R}$ for Reals. $\mathcal{X}$ is a index set and $\mathbf{x}\in\mathcal{X}$. $\langle \cdot \rangle$ for the dot product and $\|\cdot\|$ for the length of the vector.

\subsection {Convergence Analysis of SAPE-FL}

In this section we provide the proofs of theorems provided in the main manuscript. Additionally, we also discuss the convergence of our similarity-weighted global objective under both convex and non-convex settings .

\subsubsection{Proof of Theorem \ref{thm:descent}} \label{app:theorem1_proof}
\begin{proof}
	
	At each round of SAPE-FL algorithm, the server computes the new global model using Eq. (\ref{eq:global_update}). Formally, we can interpret this update as an approximate gradient (SGD) step on $\mathcal{L}$. When each client performs approximate local SGD starting from $\theta_G^t$, it is well established in the federated optimization literature \cite{fedavg_mcmahan2017communication,fedprox_li2020federated,fedamp_huang2021personalized,fedacs_chen2024personalized} that the resulting model update $\Delta_c^t$ aligns with the negative local gradient $-\nabla \mathcal{L}_c(\theta_G^t)$. In our method, the server aggregates these updates using similarity-based weights $\mathfrak{s}_c^t$, yielding an update direction that is more representative of aligned clients. At server level, the client updates are aggregated as: 
	\[\Delta^t_G = \sum_c w^t_c \Delta^t_c, \quad
	\text{ where } w_c^t= \frac{\mathfrak{s}^t_c}{\sum_j}\mathfrak{s}^t_j \] 
	and $j$ is the index for all participating clients. Therefore the inner product with the true global gradient is defined as:
	\[\langle \nabla \mathcal{L}(\theta^t_G),\Delta^t_G \rangle\ = \sum_c w^t_c \langle \nabla \mathcal{L}(\theta^t_G), \Delta^t_c \rangle\]
	
	If each client update can be approximated as $\Delta^t_c\approx -\eta \nabla \mathcal{L}_c (\theta^t_G)$, then the inner product can be written as:
	\[ \langle \nabla \mathcal{L}(\theta^t_G),\Delta^t_c \rangle = - \langle \nabla \mathcal{L}(\theta^t_G), \nabla \mathcal{L}_c(\theta^t_G)\rangle\]
	
	Therefore, 
	\[ \sum_c w^t_c \langle \nabla \mathcal{L}(\theta^t_G), \Delta^t_c \rangle \approx - \sum_c w^t_c \|\nabla \mathcal{L}_c (\theta^t_G)\|^2\]
	
	In the above equation, if the model updates are closer to the global gradient then 
	\[ \sum_c w^t_c \nabla \mathcal{L}(\theta^t_G)^\top \Delta^t_c \approx -\rho \|\nabla \mathcal{L}_c (\theta^t_G)\|^2 \]
	where $\rho>0$ quantifies alignment and reliability of updates. Therefore by the aggregated update as a stochastic descent direction on the true global objective $\mathcal{L}(\theta)$, we obtain the following condition for descent:
	\[
	\sum_c \frac{\mathfrak{s}_c^t}{\sum_j \mathfrak{s}_j^t} \nabla \mathcal{L}_c(\theta_G^t)^\top \Delta_c^t \leq -\rho \|\nabla \mathcal{L}(\theta_G^t)\|^2,
	\]
	which serves as a key step in proving one-round improvement.
	
	Each $\Delta_c^t$ in the global model update equation Eq. (\ref{eq:global_update}) can be seen as the result of running stochastic gradient descent on $\mathcal{L}_c$. For small learning rates, $\Delta_c^t \approx -\eta \nabla \mathcal{L}_c(\theta_G^t)$. If there were no weighting and full participation from all the clients i.e., $r=1$, a traditional aggregation design (FedAvg \cite{fedavg_mcmahan2017communication}) would produce:
	\[\theta_G^{t+1} \approx \theta_G^t - \eta \frac{1}{|C|}\sum_{c\in C}\nabla \mathcal{L}_c(\theta_G^t)\]
	
	Our algorithm instead uses a weighted sum $\frac{\mathfrak{s}_c^t}{\sum_i \mathfrak{s}_i^t}$ as the weight for client $c$'s gradient. If the similarity weights are aligned with the importance of each client's data or the reliability of its gradient, the weighted sum is still an unbiased or at least descent-biased estimator of the true gradient. In realistic cases, $\mathfrak{s}_c^t$ upweights clients whose gradients are in agreement with others, effectively reducing gradient “variance” or conflicting directions. The aggregated update $\Delta_G^t$ is therefore biased toward the true global gradient. Under the $L$-Lipschitz smoothness assumption of $\mathcal{L}$ and Taylor expansion (See Eq. (8) provided in the proof sketch of Theorem 4 of \cite{fedprox_li2020federated}), we have:	
	\[\mathcal{L}(\theta_G^{t+1}) \approx \mathcal{L}(\theta_G^t) + \nabla \mathcal{L}(\theta_G^t)^\top (\theta_G^{t+1}-\theta_G^t) + \frac{L}{2}|\theta_G^{t+1}-\theta_G^t|^2\] 
	and by definition of the global update,
	\[ \theta_G^{t+1} - \theta_G^t = \sum_c \frac{\mathfrak{s}_c^t}{\sum_j \mathfrak{s}_j^t} \Delta_c^t. \]
	Substituting for $\theta_G^{t+1}-\theta_G^t$, the linear term becomes 
	\[\nabla \mathcal{L}(\theta_G^t)^\top (\theta_G^{t+1}-\theta_G^t) = \sum_c \frac{\mathfrak{s}_c^t}{\sum_i \mathfrak{s}_i^t}\nabla \mathcal{L}(\theta_G^t)^\top \Delta_c^t\]
	
	Since each $\Delta_c^t$ was a descent direction on $\mathcal{L}_c$ and aligned (by regularization) with $\theta_G^t$, this sum is negative and bounded away from zero proportionally to $\|\nabla \mathcal{L}(\theta_G^t)\|^2$. Thus, $\mathcal{L}(\theta_G^{t+1}) < \mathcal{L}(\theta_G^t)$ for each round (at least in expectation), so the global model improves. Further, taking expectation over client sampling (assumed uniform for ratio $r$):
	\begin{align*}
		\mathbb{E}[\mathcal{L}(\theta_G^{t+1})] &\leq \mathcal{L}(\theta_G^t) + \mathbb{E}\left[\sum_c \frac{\mathfrak{s}_c^t}{\sum_j \mathfrak{s}_j^t} \nabla \mathcal{L}(\theta_G^t)^\top \Delta_c^t\right] \dots\\
		&\dots + \frac{L}{2} \mathbb{E}[\|\theta_G^{t+1} - \theta_G^t\|^2] \\
		&\leq \mathcal{L}(\theta_G^t) - \rho \|\nabla \mathcal{L}(\theta_G^t)\|^2 + \frac{L}{2} \mathbb{E}[\|\theta_G^{t+1} - \theta_G^t\|^2]
	\end{align*}
	
	On the other hand, the peer-regularization term in the local objective i.e., $\|\theta_c - \bar{\theta}_c^t\|^2$ term weighted by $S_c^P$ works like an anchor to a neighborhood of similar models. This is an attention mechanism that encourages each client's final model $\theta_c^{t+1}$ to not deviate too far from the average of models deemed similar i.e. $S_{cj}^t > S_{\min}$. As a result, such clients are pulled towards the cluster of more "similar" clients, preventing any single client from making an extreme update. This further bounds the heterogeneity of updates that the server aggregates.
\end{proof}

\subsubsection{Proof of Lemma 1}

\begin{proof}
	
	Under the similarity assumption on clients and peers, each local model $\theta_c^t$ (or peer model $\theta_j^t$) can be viewed as an independent unbiased estimate of the same underlying mean parameter. The peer-averaged model is defined as:
	\[
	\bar{\theta}_c^t = \frac{1}{\sum_{j \in C_k} S_{cj}} \sum_{j \in C_k} S_{cj} \, \theta_j^t.
	\]
	
	Taking expectation:
	\[\mathbb{E}[\bar{\theta}_c^t] = \frac{1}{\sum_{j \in C_k} S_{cj}} \sum_{j \in C_k} S_{cj} \, \mathbb{E}[\theta_j^t]\]
	
	Since each $\theta_j^t$ has the same covariance $\Sigma$ and $\bar{\theta}_c^t$ is the weighted average of the $|C_k|$ independent estimates (with weights $S_{cj}$), it follows that:
	\begin{equation}
		\begin{aligned}
			\mathrm{Cov}(\bar{\theta}_c^t) &= \left( \frac{1}{\sum_{j \in C_k} S_{cj}} \right)^2 \sum_{j \in C_k} S_{cj}^2 \, \mathrm{Cov}(\theta_j^t) \\ 
			&= \frac{\sum_{j \in C_k} S_{cj}^2}{\left( \sum_{j \in C_k} S_{cj} \right)^2} \Sigma
		\end{aligned}
	\end{equation}
	
	Let the error rate of a single client be $\upsilon = \mathrm{Tr}(\Sigma)$. Therefore, the error rate of the peer-averaged model (as the trace of its covariance) is given by:
	\[
	\bar{\upsilon} = \mathrm{Tr}\left( \mathrm{Cov}(\bar{\theta}_c^t) \right) = \frac{\sum_{j \in C_k} S_{cj}^2}{\left( \sum_{j \in C_k} S_{cj} \right)^2} \, \upsilon.
	\]
	
	This shows that the peer-averaged model has an error rate that is reduced by a factor of 
	\begin{equation}
		\label{eq:var_factor}
		\frac{\sum_{j \in C_k} S_{cj}^2}{\left( \sum_{j \in C_k} S_{cj} \right)^2}
	\end{equation}
		
	Now for simplicity, let the similarity scores used in the peer-averaging step are \emph{uniform}, the variance-reduction factor simplifies to the well-known result from standard averaging. We consider two common forms of uniform weighting. First, if all similarity weights are equal and unnormalized, i.e., $S_{cj} = 1$ for all $j \in C_k$, then:
	\[
	\frac{\sum_{j \in C_k} S_{cj}^2}{\left( \sum_{j \in C_k} S_{cj} \right)^2} = \frac{|C_k| \cdot 1^2}{\left( |C_k| \cdot 1 \right)^2} = \frac{|C_k|}{|C_k|^2} = \frac{1}{|C_k|}.
	\]
	Second, if the similarity scores are normalized to sum to 1 (i.e., $S_{cj} = 1/|C_k|$ for all $j \in C_k$), the same simplification holds:
	\[
	\frac{\sum_{j \in C_k} S_{cj}^2}{\left( \sum_{j \in C_k} S_{cj} \right)^2} = \frac{|C_k| \cdot \left(1/|C_k|\right)^2}{\left( \sum_{j \in C_k} 1/|C_k| \right)^2} = |C_k| \cdot \frac{1}{|C_k|^2} = \frac{1}{|C_k|}
	\]
	Hence, regardless of whether uniform weights are applied directly or normalized to sum to one, the variance-reduction factor mentioned in Eq. (\ref{eq:var_factor})	collapses to $1/|C_k|$, matching the classical reduction factor obtained from averaging $|C_k|$ i.i.d. estimators. This highlights that the similarity-weighted aggregation generalizes the standard averaging procedure, and reduces to it under uniform peer similarity.
	
	With the above simplification, the peer-averaged model is then defined as:
	\begin{equation}
		\label{eq:mod_avg_peer_model}
		\bar{\theta}_c^t = \frac{1}{|C_k|} \sum_{j \in C_k} \theta_j^t.	
	\end{equation}
	
	Taking expectation,
	\[
	\mathbb{E}[\bar{\theta}_c^t] = \frac{1}{|C_k|} \sum_{j \in C_k} \mathbb{E}[\theta_j^t] = \mathbb{E}[\theta_j^t],
	\]
	so $\bar{\theta}_c^t$ shares the same mean as each individual peer model.
	
	Since every $\theta_j^t$ has covariance $\Sigma$ and $\bar{\theta}_c^t$ is the average of $|C_k|$ independent estimates,
	\begin{equation}
		\label{eq:mod_avg_peer_cov}
	\mathrm{Cov}(\bar{\theta}_c^t) = \frac{1}{|C_k|^2} \sum_{j \in C_k} \mathrm{Cov}(\theta_j^t)
	= \frac{1}{|C_k|^2} \cdot |C_k| \cdot \Sigma = \frac{1}{|C_k|}\,\Sigma	
	\end{equation}
	
	Subsequently, the variance of the peer-averaged model is now given as:
	\[
	\bar{\upsilon} = \mathrm{Tr}\bigl(\mathrm{Cov}(\bar{\theta}_c^t)\bigr) = \frac{1}{|C_k|}\,\upsilon.
	\]
	
	Hence the peer-averaged model reduces the estimation-error variance by a factor of $1/|C_k|$, reflecting the classical variance-reduction effect of averaging independent, similarly distributed estimates. In other words, peer averaging reduces the estimation error variance, a well known benefit of using an ensemble of models. 
	
	Furthermore, by Jensen’s inequality, if each client loss $\mathcal{L}_c(\theta)$ is convex, then 
	\[\mathcal{L}_c(\bar{\theta}_c^t)\le \frac{1}{|C_k|}\sum_{j\in C_k}\mathcal{L}_c(\theta_j^t),\] 
	meaning the peer-averaged model achieves no worse (often better) loss on client $c$ than the average of individual peer losses. This justifies that $\bar{\theta}_c^t$ provides an improved baseline for personalization. Empirically and theoretically, clustering similar clients and averaging their models yields a model closer to the true task model for that cluster. \cite{ghosh2020efficient} prove that when clients are partitioned into true clusters, a clustered FL algorithm can attain a near-optimal error rate matching the minimax optimal error (analogous to pooling cluster's data and learned centrally) up to logarithmic factors \cite{ghosh2020efficient}. This suggests $\bar{\theta}_c^t$ quickly converges to the optimal cluster model with high accuracy. For linear models, one can show an exponential convergence of cluster models to the true optima with error approximately halving each iteration and in general strongly-convex settings the cluster-personalized models achieve the statistically optimal error rate. The peer averaging step thus improves the estimate of the local optimum by leveraging more data from similar clients, while filtering out dissimilar client's drift as we set $S_{cj}=0$ for those peers.	
\end{proof}

\paragraph{Peer-Average Variance Reduction under Non-Convexity} 
While the above lemma is framed under the assumption of convex objective and a unique optimal solution (same mean parameter), its core result showing the variance reduction via peer averaging  remains valid in non-convex settings under milder assumptions. Specifically, the key steps in the proof rely on the properties of expectation and covariance, not on the convexity of the underlying loss function or the existence of a unique minimizer. That is, even in non-convex scenarios where the true optimizer may not exist or may not be unique, we can still view each $\theta_j^t$ as a stochastic estimate of a \emph{latent central tendency} for a cluster of similar peers. 

In this setting, the peer-averaged model in Eq. (\ref{eq:mod_avg_peer_model}) remains a well-defined random variable, and its variance is still governed by the same formula for the variance of independent estimators mentioned in Eq. (\ref{eq:mod_avg_peer_cov}). 
This results in a reduced total variance:
\[
\mathrm{Tr}\left( \mathrm{Cov}(\bar{\theta}_c^t) \right) = \frac{1}{|C_k|^2} \sum_{j \in C_k} \mathrm{Tr}(\Sigma_j) \le \frac{1}{|C_k|} \max_j \mathrm{Tr}(\Sigma_j).
\]

Therefore, the statistical benefit of peer averaging holds irrespective of whether the estimators $\theta_j^t$ converge to a true minimum or just to a representative region in the non-convex search space. Moreover, non-convex FL literature (e.g., \cite{scaffold_karimireddy2020scaffold}, \cite{fedprox_li2020federated}) commonly leverages similar averaging steps without requiring uniqueness of the optimum. The averaging mechanism still acts as a "variance smoother" and  helping guide the optimization toward flatter regions or stationary points with better generalization. Hence, the lemma remains applicable, interpreted as a \emph{statistical variance reduction tool} rather than a strict estimator of a unique minimizer.

\subsubsection{Proof of Theorem \ref{thm:pers_obj}}
Each client update performs a proximal step over a smooth, composite objective anchored to both the global and peer-averaged models. With bounded similarity-weighted averaging and fixed regularization, the joint model converges to a critical point efficiently. For the sake of completeness, we define the joint objective function across all clients $c\in C$ is given by:

\begin{align*}
	\mathcal{L}(\theta_G, \{\theta_c\}) = \sum_c \mathcal{L}_c(\theta_c) + & \sum_c \mathcal{S}_c^P(\theta_c, \bar{\theta}_c^t) + \cdots \\ & \cdots \sum_c \mathcal{S}_c^G(\theta_c, \theta_G^t)
\end{align*}

The averaging of peer models is performed using a weighted combination of those satisfying the condition $s_{ij} > S_{\min}$, which ensures that each $\bar{\theta}_c^t$ is closer to clients with high similarity. The update direction benefits from both model similarity and output similarity, accelerating consensus within high-similarity groups. Under smoothness, $\mathcal{L}$ decreases each iteration, and in convex cases, it converges to a global minimum. 

As per Theorem 1, the global model $\theta_G^t$ also converges through similarity-weighted aggregation of updates, yielding a stable reference for collaboration across all clients. Alongside restricting the client model $\theta_c^t$ to drift too much from such global reference model $\theta_G^t$, the structure of $\mathcal{L}$ also enforces an implicit clustering effect where the clients with strong mutual similarity weights contribute more significantly to each other's peer-averaged models. The Laplacian-like regularization promotes consensus in peer neighborhoods. This results in clients converging also toward a shared local optimum within similarity-based groups. Over iterations, the contribution of highly similar peers becomes more influential, driving each $\bar{\theta}_c^t$ toward a cluster-wide personalized optimum.

Each client's local update is a proximal gradient step on the objective $\mathcal{L}$. At each round $t$, clients perform a proximal step:
\[
\theta_c^{t+1} = \theta_c^t - \eta \left( \nabla \mathcal{L}_c(\theta_c^t) + 2\mathcal{S}_c^G (\theta_c^t - \theta_G^t) + 2\mathcal{S}_c^P (\theta_c^t - \bar{\theta}_c^t) \right),
\]
where $\eta$ is the learning rate.

For brevity, we define, 
\[ \mathfrak{g}^t_c = \nabla \mathcal{L}_c(\theta_c^t) + \mathcal{S}_c^G (\theta_c^t - \theta_G^t) + \mathcal{S}^P_c (\theta_c^t - \bar{\theta}_c^t)\]

Under Assumption 1 and as per the definitions provided in Theorem 2 (Appendix A.5) of \cite{pfedme_t2020personalized}, we have, 
\[\mathcal{L}_c(\theta_c^{t+1}) = \mathcal{L}_c(\theta_c^{t}) - \eta \|\mathfrak{g}_c^t \|^2 + \frac{L\eta^2}{2} \|\mathfrak{g}_c^t \|^2 \]

Using this we write the descent in the objective as:
\begin{align*}
	\mathcal{L}(\theta_G^{t}, \{\theta_c^{t+1}\}) - \mathcal{L}(\theta_G^t, \{\theta_c^t\}) & \leq -\eta \sum_c \left\| \mathfrak{g}^t_c\right\|^2+ \cdots\\
	&\cdots \frac{L\eta^2}{2} \sum_c \left\| \mathfrak{g}^t_c\right\|^2
\end{align*}

Further, 
\[\mathcal{L}(\theta_G^{t}, \{\theta_c^{t+1}\}) \leq \mathcal{L}(\theta_G^t, \{\theta_c^t\}) - \eta \left( 1-\frac{L\eta}{2}\right) \left\| \mathfrak{g}^t_c\right\|^2\]

In the above equation ensuring $\left( 1-\frac{L\eta}{2}\right)>0$ and picking small learning rate $\eta \le \frac{1}{2L}$, we obtain monotonic decrease:
\[
\mathcal{L}(\theta_G^{t+1}, \{\theta_c^{t+1}\}) \leq \mathcal{L}(\theta_G^t, \{\theta_c^t\}) - \frac{\eta}{2} \sum_c \|\mathfrak{g}^t_c\|^2.
\]

If $\theta_G^0$ and $\theta_c^0$ are the initial models we started with, then summing over $T$ rounds gives:
\[
\mathcal{L}(\theta_G^{T}, \{\theta_c^{T}\}) \leq \mathcal{L}(\theta_G^0, \{\theta_c^0\}) - \frac{\eta}{2} \sum_{t=1}^T \sum_c \|\mathfrak{g}^t_c\|^2.
\]

By rearranging, we get, 
\[
\frac{1}{T} \sum_{t=1}^{T} \sum_c \|\nabla_{\theta_c} \mathcal{L}(\theta_c^t)\|^2 \le \frac{2(\mathcal{L}(\theta_G^{0}, \{\theta_c^{0}\}) - \mathcal{L}(\theta_G^{T}, \{\theta_c^{T}\})}{\eta T},
\]
Therefore, for a target accuracy $\epsilon$,
\[
\frac{1}{T} \sum_{t=1}^{T} \mathbb{E}[\|\nabla \mathcal{L}(\theta^t_c)\|^2] \le \frac{\Sigma}{|C_k|\epsilon^2}.
\] 

In standard FL or SGD without peer-averaging, the variance of the gradient estimate is assumed constant (or just bounded by $\Sigma$), which leads to the usual bound:
\[T=\mathcal{O}\left(\frac{\Sigma}{\epsilon^2}\right)\]

However, because of peer-averaging in our proposed framework and thus the variance reduction property of $\bar{\theta}_c^t$ (see proof of Lemma 1), we observe a tightened convergence rate. Specifically, let $\Sigma$ denote the covariance of a single local model $\theta_c^t$. Then the peer-averaged estimator yields:
\[
\mathrm{Cov}(\bar{\theta}_c^t) \leq \frac{1}{\sum_j S_{cj}^t} \Sigma \leq \frac{1}{|C_k|} \Sigma,
\]
where $|C_k|$ is the number of similar peers. As per Lemma 1, the reduced variance is given as: 
\[\mathrm{Var}(\bar{\theta}^t_c) = \frac{1}{|C_k|}\Sigma \]

Practically $|C_k|>\!\!>1$ and $\mathcal{O}(\Sigma)=1$, then the resulting bounds are given as:
\[ T = \mathcal{O}\left(\frac{\text{Var}(\bar{\theta})}{\epsilon^2}\right) \leq \mathcal{O}\left(\frac{\Sigma}{|C_k| \cdot \epsilon^2}\right) < \mathcal{O}\left(\frac{1}{\epsilon^2}\right)
\]

This tighter bound implies reduced variance in client updates, leading to a faster decrease in the global objective established in Theorem 1:
\[
\mathbb{E}[\mathcal{L}(\theta_G^{t+1})] \leq \mathcal{L}(\theta_G^t) - \rho \|\nabla \mathcal{L}(\theta_G^t)\|^2 + \frac{L}{2} \mathbb{E}[\|\theta_G^{t+1} - \theta_G^t\|^2].
\]
Thus, our method requires fewer rounds to reach an $\epsilon$-stationary solution demonstrating that peer-regularized personalization yields faster convergence due to statistical averaging across similar peers.

\begin{theorem}[Global Convergence under Strong $\mu$-Convexity]

Suppose a client's loss function $\mathcal{L}_c(\theta)$ is $\mu$-strongly convex and $L$-smooth (i.e. has $L$-Lipschitz gradients) as per Assumptions 1-3. The global model is updated by similarity-weighted aggregation as per Eq. (\ref{eq:global_update}). Then there exists a unique global minimizer $\theta_G^\star$ that minimizes the similarity-weighted global objective and in particular achieves linear convergence: for a sufficiently small constant step size, 
\[\theta_G^t-\theta^\star \le (1-\mu\eta)^t \|\theta_G^0 - \theta_G^\star \|^2\] 
Equivalently, the optimality gap decays: $F(\theta_G^t) - F(\theta_G^\star) = \mathcal{O}((1-\mu\eta)^t)$. 
\end{theorem}
\begin{proof}

Let us assume full client participation i.e. all $C$ clients are selected per round and similarity scores $s_c^t$ is non-negative and bounded as $\sum_{c}s_c^t = 1$ for each round $t$. For brevity, let us define the similarity-weighted global objective as: 
\[F(\theta) := \sum_c s_c \;\mathcal{L}_c(\theta)\] 
for appropriate fixed weights ${s_c}$. The unique minimizer $\theta_G^\star$ is defined as:  \[\theta_G^\star:=\arg\min_{\theta}F(\theta)\] 

At each round, the server computes the new global model (using Eq. (\ref{eq:global_update})) and formally, we can interpret this update as a gradient descent step on $F$, up to the influence of the peer-regularization. In particular, excluding the regularizers, we have on each round:

\[\theta_G^{t+1} = \theta_G^t - \eta \sum_c s_c^t \nabla \mathcal{L}_c(\theta^t_G) \approx \theta^t_G - \eta \;\nabla F(\theta_G^t)\] 
because $\nabla F(\theta) = \sum_{c}{s}_c\,\nabla \mathcal{L}_c(\theta)$. If $s_c^t$ varies with $t$, each step still follows the aggregated gradient of the objective; for simplicity we retain $s_c^t$ constant, since the analysis extends similarly as long as weights are bounded.

The client peer-regularized local loss can be written as
\begin{equation}
	\label{eq:peer_regularized_loss}
	\tilde{\mathcal{L}}_c^t(\theta) = \mathcal{L}_c(\theta) + R_c^t(\theta),
\end{equation} 
where $R_c^t$ is a regularizer term that penalizes discrepancy from peer models (via $\bar{\theta}_c^t$ defined in Eq. (\ref{eq:avg_peer_model})) or from the global model $\theta_G^t$. The regularizer $R_c^t$ only makes $\tilde{\mathcal{L}}_c^t$ more strongly convex. For example, in the convergence analysis of FedProx \cite{fedprox_li2020federated} $\tilde{\mathcal{L}}_c^t(\theta)$ mentioned in Eq. (\ref{eq:peer_regularized_loss}) is set as: 
\[R_c^t(\theta)=\frac{\lambda}{2}|\theta-\theta_G^t|^2,\] 
whose gradient at $\theta=\theta_G^t$ is 0. However, in \myfls, regularizer $R_c^t(\theta)$ does not vanish at $\theta = \theta_G^t$ due to the peer term ($\bar{\theta}_c^t$), we can control its influence under reasonable assumptions so that it is small enough to preserve descent guarantees in the global convergence analysis. We argue that the total regularizer gradient $\nabla R_c^t(\theta)$ can still be treated as a bounded perturbation to the local optimization path. Specifically, the gradient remains small in magnitude if $\bar{\theta}_i^t$ is close to $\theta_G^t$ and similarity-based filtering ensures that $\bar{\theta}_i^t$ is an average over similar clients with relatively aligned models. 

In \myfls, at server level, the client updates are aggregated as: 
\[\Delta^t_G = \sum_c w^t_c \Delta^t_c, \quad
\text{ where } w_c^t= \frac{\mathfrak{s}^t_c}{\sum_j}\mathfrak{s}^t_j \] 
and $j$ is the index for all participating clients. 

Therefore, at the start of each round $t$ i.e, when $\theta=\theta_G^t$, the update direction 
\[\Delta_c^t = -\eta,\nabla \tilde{\mathcal{L}}_c^t(\theta_G^t)\] can be approximated as \[\Delta_c^t \approx -\eta,\nabla \mathcal{L}_c(\theta_G^t)\] since $\nabla R_c^t(\theta_G^t)=0$ or is negligibly small. In other words, the similarity weighting and peer regularization do not hinder the descent direction; they effectively keep client updates closer to the true global gradient, mitigating drift. We can thus analyze the global iteration as a gradient descent on $F(\theta)$. 

Given that $F$ is $\mu$-strongly convex and has Lipschitz gradient $L$, it is well-known that gradient descent enjoys a linear convergence rate \cite{fedprox_li2020federated}. Fore brevity, Let $g^t := \nabla F(\theta_G^t) = \sum_c s_c\,\nabla \mathcal{L}_c(\theta_G^t)$ be the true global gradient. By $L$-smoothness assumption on $F$ and the taylor expansion, we have:
\begin{equation}
	F(\theta_G^{t+1} \le F(\theta_G^t) + \langle \nabla F(\theta^t_G), \theta^{t+1}_G - \theta_G^t \rangle + \frac{L}{2} \|\theta_G^{t+1} - \theta_G^t\|^2
\end{equation}  
Substituting $\theta_G^{t+1} = \theta_G^t - \eta\;g^t$, we get
\begin{equation}\label{eq:ineq_ft1_t}
	F(\theta_G^{t+1}) \le F(\theta_G^t) - \eta \|g^t\|^2 + \frac{L\eta^2}{2} \|g^t\|^2.
\end{equation}
For $\eta \le 1/L$ Eq. (\ref{eq:ineq_ft1_t}) simplifies to
\begin{equation} \label{eq:simplified_each_round}
F(\theta_G^{t+1}) \le F(\theta_G^t) - \frac{\eta}{2}\,\|g^t\|^2	
\end{equation}
Since $F$ is $\mu$-strongly convex, we use the definition 
\[ F(\theta_G^\star) \ge F(\theta_G^t) + \langle \nabla F(\theta_G^t), \theta^\star_G - \theta^t_G \rangle\ + \frac{\mu}{2} \|\theta_G^\star - \theta_G^t|^2 \]
Rewriting, 
\[ F(\theta_G^t) - F(\theta_G^\star) \le \langle \nabla F(\theta_G^t), \theta^t_G - \theta^\star_G \rangle\ - \frac{\mu}{2} \|\theta_G^\star - \theta_G^t\|^2 \]
By definition, 
\[\langle \nabla F(\theta_G^t), \theta^t_G - \theta^\star_G \rangle \le \|\nabla F(\theta^t_G)\| \cdot \|\theta_G^t - \theta^\star_G\|\]
Therefore, 
\begin{equation} \label{eq:in_prod} 
	F(\theta_G^t) - F(\theta_G^\star) \le \|\nabla F(\theta^t_G)\| \cdot \|\theta_G^t - \theta^\star_G\| - \frac{\mu}{2} \|\theta_G^\star - \theta_G^t\|^2
\end{equation}
We assume $F$ is $L$-smooth (Assumption 1), therefore:
\[ \|\nabla F(\theta^t_G) - \nabla F(\theta_G^\star)| \le L \|\theta_G^t -\theta_G^\star\|\]
But since $\theta^\star_G$ is a minimizer, $\nabla F(\theta^\star_G) = 0$, so:
\begin{equation}\label{eq:nab_f_theta}
	\|\nabla F(\theta^t_G)\|\le L\|\theta_G^t-\theta_G^\star\| \implies \|\theta_G^t-\theta_G^\star\| \ge \frac{\|\nabla F(\theta^t_G)\|}{L}
\end{equation}
Using Eq. (\ref{eq:nab_f_theta}) in Eq. (\ref{eq:in_prod}) we get,
\begin{equation}
\begin{aligned}
	F(\theta_G^t) - F(\theta^*) 
	&\leq \|g^t\| \cdot \frac{\|g^t\|}{L} - \frac{\mu}{2} \left(\frac{\|g^t\|}{L}\right)^2 \\
	&= \frac{\|g^t\|^2}{L} - \frac{\mu}{2} \cdot \frac{\|g^t\|^2}{L^2} \\
	&= \|g^t\|^2 \left( \frac{1}{L} - \frac{\mu}{2L^2} \right)
\end{aligned}
\end{equation}

Invert this inequality to get a lower bound on $\|g^t\|^2$:
\begin{equation}
g^t\|^2 \geq \frac{F(\theta_G^t) - F(\theta^*)}{\left(\frac{1}{L} - \frac{\mu}{2L^2}\right)} = \left(\frac{L^2}{L - \mu/2}\right) \left(F(\theta_G^t) - F(\theta^*)\right)
\end{equation}
\[\|g^t\|^2 = \|\nabla F(\theta_G^t)\|^2 \ge 2\mu(F(\theta_G^t)- F(\theta_G^\star))\] 
Plug this in Eq. (\ref{eq:simplified_each_round}), we get, 
\[F(\theta^{t+1}_G) \le F(\theta_G^t) - \eta\mu (F(\theta^t_G) - F(\theta_G^\star))\]
On further simplification, we get, 
\[F(\theta^{t+1}_G) -F(\theta_G^\star) \le(1-\eta\mu) F(\theta_G^t) -F(\theta_G^\star))\]
This linear convergence in the equation above proves that each round reduces the optimality gap by a factor of $(1 - \eta \mu)$. Since $0 < \eta \mu < 1$, this optimality further shrinks. 
By induction, $\forall t$, 
\[F(\theta^{t}_G) -F(\theta_G^\star) \le(1-\eta\mu)^t F(\theta_G^0) -F(\theta_G^\star))\]
Now, from Assumption 2, 
\[F(\theta^{t}_G) -F(\theta_G^\star) \ge \frac{\mu}{2}\|\theta_G^t - \theta_G^\star\|^2\]
\[\frac{\mu}{2}\|\theta_G^t - \theta_G^\star\|^2 \le (1-\eta\mu)^t (F(\theta_G^0) -F(\theta_G^\star))\]
Rewriting, 
\[\|\theta_G^t - \theta_G^\star\|^2 \le \frac{2}{\mu}(1-\eta\mu)^t (F(\theta_G^0) -F(\theta_G^\star))\]
Since $F(\theta_G^0) - F(\theta_G^\star) \le \frac{L}{2} |\theta_G^0 - \theta_G^\star|^2$ ($L$-smoothness assumption), we can upper bound as: 
\[\|\theta_G^t - \theta_G^\star\|^2 \le (1-\eta\mu)^t (\theta_G^0 - \theta_G^\star)\]
which implies the parametric error $|\theta_G^t - \theta^\star_G|^2$ contracts linearly and thus establishing linear convergence.

The peer-regularization term $(R_c^t(\theta))$ assists convergence by reducing client-drift and effectively increasing the strong convexity of local objectives. The above result shows SAPE-FL converges to the exact minimizer $\theta^\star_G$ when using an appropriate constant $\eta$ or a decaying learning-rate schedule. In contrast, as mentioned in \cite{li2019convergence}, FedAvg on heterogeneous data may fail to reach the exact optimum unless $\eta$ is decayed, incurring an $O(\eta)$ steady-state error if $\eta$ is fixed. The similarity-weighted aggregation and peer regularization in SAPE-FL overcomes this issue by focusing updates on similar clients and coupling their objectives, thereby stabilizing and accelerating convergence.
\end{proof}

\begin{theorem}[Global Convergence under Non-Convexity]
Let the similarity-weighted aggregated objective is defined as $F(\theta^t_G):=\sum_c s_c^t \mathcal{L}_c(\theta^t_G)$. Suppose the stochasticity in local updates is bounded and  gradients are uniformly bounded $|\nabla \mathcal{L}_c(\theta)| \le G$ as per  Assumption \ref{assum:grad_bound}. For a fixed small stepsize $\eta \le 1/L$, after $T$ rounds \myfls satisfies:
\[ \min_{0\le t< T} \mathbb{E}\left[ \|\nabla F(\theta_G^t)\|^2 \right] \le \frac{2(F(\theta_G^0)- F_{\inf})}{\eta T} + \eta L \sigma^2,\]
where $F_{\inf}$ is a lower bound on $F$ and $\sigma^2$ is some finite variance. \myfls converges to a stationary point in expectation as $T\to\infty$.  Equivalently, \myfls finds an $\epsilon$-approximate stationary point i.e. $\mathbb{E}|\nabla F(\theta)|\le \epsilon$ in $\mathcal{O}(1/\epsilon^2)$ rounds in the worst case.
\end{theorem}

\begin{proof}

Without convexity, we cannot guarantee convergence to a global minimum, but we can show the global model descends toward a stationary point where the gradient is zero. The proof leverages the $L$-smoothness of each $\mathcal{L}_c$ and the fact that the global update is a weighted gradient step. We again interpret the update as (an approximate) gradient descent on the instantaneous aggregated loss:
\[F(\theta) = \sum_c \mathfrak{s}_c^t \mathcal{L}_c(\theta)\]
For notational simplicity, let us again define $g^t$ as:  
\[g^t:= \nabla F(\theta_G^t) = \sum_c \mathfrak{s}_c^t \nabla \mathcal{L}_c(\theta_G^t)\]
Also let us define $\tilde{g}^t$ as
\[\tilde{g}^t := -\frac{1}{\eta}\sum_c \mathfrak{s}_c^t \Delta_c^t\]
be the effective global gradient that the server uses (so ideally $\tilde{g}^t \approx g^t$). The peer regularization again helps align $\tilde{g}^t$ with $g^t$ by discouraging any client from drifting too far from the global model or its peers. We can regard $\tilde{g}^t$ as an unbiased estimator of $g^t$ (since $\mathbb{E}[\Delta_c^t] \approx -\eta \nabla \mathcal{L}_c(\theta_G^t)$ and assuming full client participation and unbiased local gradients). Therefore, $\mathbb{E}[\tilde{g}^t \mid \theta_G^t] = g^t$, and the variance between them is bounded as (Assumption \ref{assum:grad_bound}):
\[\mathbb{E}[|\tilde{g}^t - g^t|^2] \le \sigma^2\]
Using $L$-smoothness of $F$ and for any $\eta \le 1/L$:
\begin{equation*}
	\footnotesize
		\begin{aligned}
		\mathbb{E}\big[F(\theta_G^{t+1}) \mid \theta_G^t\big] &\le F(\theta_G^t) + \mathbb{E}\big[\langle \nabla F(\theta_G^t),\theta_G^{t+1}-\theta_G^t \rangle \mid \theta_G^t\big] \cdots \\ 
		& \qquad \cdots + \frac{L}{2}\mathbb{E}\big[\|\theta_G^{t+1}-\theta_G^t\|^2 \mid \theta_G^t\big] \\
		&= F(\theta_G^t) - \eta,\langle g^t,\mathbb{E}[\tilde{g}^t \mid \theta_G^t]\rangle \cdots \\ & \qquad \cdots + \frac{L\eta^2}{2}\mathbb{E}\big[\|\tilde{g}^t\|^2 \mid \theta_G^t\big] \\
		&= F(\theta_G^t) - \eta\,\|g^t\|^2 + \frac{L\eta^2}{2},\mathbb{E}\big[\|\tilde{g}^t\|^2 \mid \theta_G^t\big]
	\end{aligned}
\end{equation*}
Since $\mathbb{E}|\tilde{g}^t|^2 = |g^t|^2 + \mathbb{E}|\tilde{g}^t - g^t|^2 \le |g^t|^2 + \sigma^2$, we get:
\[\mathbb{E}[F(\theta_{G}^{t+1})] \le \mathbb{E}[F(\theta_{G}^{t})] - \frac{\eta}{2}\mathbb{E}\left[\|g^t\|^2\right] + \frac{L\eta^2}{2} \sigma^2\]

Rearrange this to obtain a lower bound on $\mathbb{E}[|g^t|^2]$:
\begin{equation}\label{eq:gt_descent}
\frac{\eta}{2}\,\mathbb{E}\!\big[\|g^t\|^2\big] \;\le\; \mathbb{E}\big[F(\theta_G^t)\big] - \mathbb{E}\big[F(\theta_G^{t+1})\big] + \frac{L\eta^2}{2}\,\sigma^2	
\end{equation}

Given that $F(\theta_G^t)$ is a convex combination of the ${\mathcal{L}_c}$ at $\theta_G^t$. Since all $\mathcal{L}_c$ are bounded, we assume lower bounds on $F$ as $F_{\inf}>-\infty$. Moreover, the descent condition in Eq. (\ref{eq:gt_descent}) indicates that if the gradient norm $\|g^t\|$ stays away from 0, the term $F^t(\theta_G^t) - F^t(\theta_G^{t+1})$ will be significantly positive, meaning the global loss decreases by a non-trivial amount. Summing Eq. (\ref{eq:gt_descent}) over $t=1$ to $T$:
\begin{equation}
	\label{eq:sum_gt_descent}
	\begin{aligned}
	\frac{\eta}{2} \sum_{t=0}^{T-1} \mathbb{E}\left[ \|g^t\|^2\right] \le & \mathbb{E}\left[F(\theta_G^0)\right] - \mathbb{E}\left[F(\theta_G^{T-1})\right] + \frac{L\eta^2}{2} T \sigma^2\\
\end{aligned}
\end{equation}
Substituting $F(\theta_G^{T-1}) \ge F_{\inf}$ and dividing both sides by $\frac{\eta}{2}T$ gives:
\begin{equation}\label{sum_final_gt_descent}
\frac{1}{T}\sum_t\mathbb{E}\left[ \|\nabla F(\theta_G^t)\|^2 \right] \le \frac{2(F(\theta_G^0)- F_{\inf})}{\eta T} + \eta L \sigma^2
\end{equation}

This is exactly the stated bound. It implies that the minimum gradient norm (in expectation) over the first $T$ rounds is bounded by the same RHS (since the minimum is less than the average). In particular, if we set $\eta$ to a small constant or use a decaying $\eta$, the term $L\eta,\sigma^2$ can be made negligible, giving 
\[\min_{0\le t<T}\mathbb{E}\left[||\nabla F^t(\theta_G^t)||^2\right] \approx \frac{2(F(\theta_G^0)-F_{\inf})}{\eta T}\]

Thus, as $T\to\infty$, the minimum expected gradient norm approaches zero. Now, suppose we want the minimum expected gradient norm across all rounds to be less than or equal to $\epsilon$, then from Eq. (\ref{sum_final_gt_descent}),
\[\min_{0\le t<T}\mathbb{E}\left[\|\nabla F^t(\theta_G^t)\|^2\right] \le \frac{2(F(\theta_G^0)-F_{\inf})}{\eta T}+L\eta\sigma^2 \le\epsilon^2\]

In the above equation the first term in the RHS decreases as $1/T$, and the second term increases linearly with $\eta$, so choosing $\eta=\frac{1}{\sqrt{T}}$ balances the two terms, giving the following bounds. 

\[\min_{0\le t<T}\mathbb{E}\left[\|\nabla F^t(\theta_G^t)\|^2\right] \le \mathcal{O}\left(\frac{1}{\sqrt{T}}\right)\]

Solving $\mathcal{O}(1/\sqrt{T}) \le \epsilon^2$ gives $T=\mathcal{O}(1/\epsilon^2)$. Therefore, \myfls guarantees convergence to an $\epsilon$-stationary point within $\mathcal{O}(1/\epsilon^2)$ communication rounds. The above result shows that SAPE-FL, despite personalization, will not diverge on non-convex objectives; instead it will approach a stationary solution. In fact, the client selection in SAPE-FL can improve the convergence stability by filtering out dissimilar updates and thus providing an additional safeguard against high-variance updates. 
\end{proof}

\subsection{Additional Experimental Details}

\subsubsection{Baselines}
Since FL litearture is fairly broad, we restrict our comparison specifically to methods that exhibit fundamental similarities and conceptual relevance. For the sake of completeness in our experiments, we have also compared the performance of the local-only trained models. We use the following baselines:

\begin{itemize}
	\item \textbf{FedAvg} \cite{fedavg_mcmahan2017communication}: The classic federated learning approach that aggregates client models by naively averaging their weights.
	\item \textbf{PerFedAvg} \cite{perfedavg_fallah2020personalized}: A meta-learning-based personalized federated learning method that optimizes an initial global model to facilitate effective local adaptation at each client.
	\item \textbf{FedProx} \cite{fedprox_li2020federated}: Enhances FedAvg by introducing a proximal term to mitigate local drift during training, thereby managing heterogeneity in federated settings.
	\item \textbf{DittoFL} \cite{ditto_fl}: Employs a personalized regularization term (analogous to proximal term in FedProx) that explicitly balances local fitting with adherence to the global model, offering a strong baseline for personalized federated learning.
	\item \textbf{FedACS} \cite{fedacs_chen2024personalized}: Introduces attention-based client selection and aggregation to provide personalized federated learning, considering similarity measured  between clients based on their model updates.
	
	\item \textbf{LCFed} \cite{zhang2025lcfed}: Proposes a lightweight client clustering framework that groups clients based on cosine similarity of their model updates to enable efficient and personalized federated training.
	
\end{itemize}

\subsubsection{Datasets} 
We evaluate the performance of our proposed \myfls framework on both synthetic and standard real-world multi-class classification problems using standard datasets used in the federated learning community.

\textbf{Synthetic Datasets} To simulate synthetic classification data, we utilize the \textit{make\_classification()} function from the scikit-learn library \cite{sklearn}, which generates a multi-class dataset by sampling from multivariate Gaussian distributions. Mathematically, for a specified number of classes $\mathcal{C}=10$, number of features $d=20$, number of informative features ($n_\text{info} = 0.75 \times d$) and total samples $N=10000$, we construct $\mathcal{C}$ Gaussian clusters in a feature space $\mathbb{R}^{d}$. Each cluster corresponds to a class label and is characterized by a randomly generated mean vector and a covariance structure that governs feature correlation. Informative features are those that directly influence class separation, while redundant features are created as linear combinations of the informative ones. Non-informative (noise) features are sampled independently from a Gaussian distribution. To add the statistical heterogeneity we randomly change the cluster of data points with probability = $0.4$ i.e., change the class labels during the data generation process and then split the dataset among all the clients. 

\textbf{Real-world Datasets}  We use the following publicly available classification datasets commonly utilized in federated learning literature.
\begin{itemize}
	\item \textit{Human Activity Recognition Datasets}: In this task the participants were engaged in distinct physical activities such as walking, sitting, and standing, which serve as the target class labels for classification. The raw input features are sourced from the smartphone’s built-in accelerometer and gyroscope, capturing tri-axial motion and orientation signals over time. The objective is to predict the human activity based on their accelerometer and Gyroscope data. We use publicly available HAR dataset \cite{har_dataset_2013} and AccGyro Dataset \cite{accgyro_dataset_2022} for predicting the human activity. 
	
	\item \textit{Image Classification Datasets}:  We evaluate our models on four widely used image classification datasets: MNIST, Fashion-MNIST (FMNIST), CIFAR-10, and CIFAR-100. MNIST consists of 70,000 grayscale images of handwritten digits $(0–9)$, each of size $28\times28$ pixels, serving as a foundational dataset for digit recognition tasks. EMNIST extends MNIST dataset by incorporating handwritten uppercase and lowercase letters, resulting in more complex variants with up to $62$ classes. Fashion-MNIST retains the $28×28$ grayscale format but replaces digits with images of clothing items from $10$ fashion categories, introducing greater visual diversity and classification difficulty. In contrast, CIFAR-10 and CIFAR-100 comprise $32\times32$ color (RGB) images with $10$ and $100$ classes, respectively, capturing objects such as animals, vehicles, and household items. Together, these datasets span a range of tasks from digit and character recognition to object and fine-grained image classification, enabling comprehensive evaluation across varying input modalities and classification granularities.

\end{itemize}

\subsubsection{Data Partitions} 

To simulate the statistical heterogeneity commonly observed in federated learning environments, we partition the global dataset among clients using a Dirichlet distribution-based sampling strategy. Specifically, we draw client-specific label distributions from a Dirichlet distribution parameterized by a concentration parameter $\kappa$, where lower values of $\kappa$ produce higher degrees of heterogeneity by encouraging label imbalance across clients. This setup ensures that each client possesses a unique and potentially skewed distribution of data, capturing the non-IID nature of federated settings. We set the concentration parameter $\kappa=0.3$ to systematically control the severity of data skew, enabling the evaluation of \myfls framework's robustness under varying levels of prior probability shift and quantity imbalance. 

\subsection{Parameters of \myfl} 
We have set the number of clients to $100$ in our experiments using synthetic, human activity recognition and image classification datasets. We split the data at each client in the $80-20\%$ train-test ratio. We have set the participants ratio $r$ to $70\%$, i.e., at each communication round $70\%$ of the clients participate in the federated learning process. We have set the influence factor mentioned in Eq. (\ref{eq:def_reg_strength}) to $\delta = 0.5$. We have implemented all the methods considered using Torch 2.7.0.

\subsubsection{Performance Metrics and Client Models} 
The empirical performance of our proposed and the competing methods is evaluated by measuring the mean testing accuracy i.e., the average (along with the standard deviation) test accuracies computed across all participating clients indicating overall classification performance. We also report average (along with the standard deviation) AUC scores, a robust measure particularly valuable in scenarios with imbalanced data distributions, assessing the model's capability to distinguish between different classes. Although our proposed method is applicable across different model architectures, for synthetic data and human activity recognition (HAR) tasks, we employ a five-layer multi-layer perceptron (MLP) consisting of an input layer followed by three hidden layers of sizes 1024, 512, and 256, each with ReLU activation, and a final output layer matching the task-specific number of classes. For classification tasks with image datasets such as CIFAR-10, MNIST, EMNIST, and FEMNIST, we utilize lightweight Convolutional Neural Network (CNN) architectures composed of two convolutional layers, followed by max-pooling, and two fully connected layers.


\subsubsection{Numerical Stability in Peer and Global Anchoring}

Due to the use of similarity thresholding and nonnegativity truncation in the computation of peer similarities (via $\nu$), it is possible that, for a given client $c$ at communication round $t$, no peer model satisfies the threshold criterion $S_{cj}^t > S_{\min}$. This situation may arise particularly in early training rounds, under severe data heterogeneity, or when client models are totally different. In such cases, the peer set becomes empty, resulting in a zero normalization constant in the peer-averaged model computation. To ensure numerical stability and maintain a well-defined optimization objective, we adopt a conservative fallback strategy by disabling the peer regularization term for that round, i.e., setting $S_c^P = 0$. This effectively removes peer influence while allowing the client to rely solely on the global anchor, thereby avoiding arbitrary parameter averaging and preserving stable personalization behavior.

Additionally, a similar case may arise during server-side aggregation when the aggregate similarity mass becomes zero, i.e., $\sum_{c} s_c^t = 0$, where $s_c^t$ denotes the cumulative similarity score associated with client $c$. This situation can occur when similarity scores are truncated to zero due to the non-negativity constraint imposed on the cosine similarity function, or when participating clients are highly dissimilar and fail to exhibit sufficient alignment with their peers. To ensure numerical stability and prevent arbitrary updates under such degenerate conditions, we adopt a conservative fallback strategy and skip the global model update for that round by setting $\theta_G^{t+1} \leftarrow \theta_G^t$. This approach avoids incorporating unreliable or uninformative updates while maintaining stable training dynamics, particularly in early rounds or under extreme data heterogeneity.

\subsubsection{Scalability of \myfl}

We would like to clarify that while the formulation in Eq. (8) defines peer-averaging over all clients, however in practice each client interacts only with a small sampled subset of peers (refer to step 2 of Algorithm 1) rather than all $N$ clients, keeping the complexity $\mathcal{O}(C)$ where $C\ll N$. The additional computation arises from a few forward passes of peer models on local mini-batches, which is minor compared to local training. In our experiments, this design yields $3-5\%$ accuracy and $2-4\%$ AUC improvements, demonstrating a favorable accuracy-to-cost trade-off. In contrast to FedProx, which uses a single global anchor, SAPE-FL adaptively exploits meaningful peer information, reducing variance and negative transfer.

\subsubsection{Cosine Similarity based Scores} 

We have adopted a modified (clipped) cosine similarity for comparisons with the anchor models, because it provides a scale-invariant measure of alignment that remains numerically stable across heterogeneous clients. Unlike Euclidean or kernel-based distances, cosine similarity focuses on directional coherence rather than magnitude, making it robust to local learning-rate or norm differences common in FL. Using cosine similarity measure for both weights and outputs also ensures consistent interpretation of alignment across structural and functional spaces, simplifying the adaptive weighting mechanism in Eq. (6). 

\section{Additional Experiments and Ablation Study}

\subsection{Evolution of Client Models Over Rounds}

To illustrate the effectiveness of our proposed framework in improving client-level performance over communication rounds, we conducted an additional experiment using the synthetic dataset described earlier. In this setup, we retain most experimental settings from the main manuscript, except we reduce the number of clients to $C=20$ and the number of communication rounds to $100$. Figure \ref{fig:client_accuracy_trajectory} shows the accuracy trajectories $5$ randomly selected clients. We observe a consistent upward trend in accuracy across all the clients, with most achieving stable performance after a sufficient number of communication rounds. These results demonstrate that \myfls framework enables steady local improvement by leveraging personalized and peer-regularized training, even in settings with increased client population and extended training duration.

\begin{figure}[t]
	\centering
	\includegraphics[width=0.95\columnwidth]{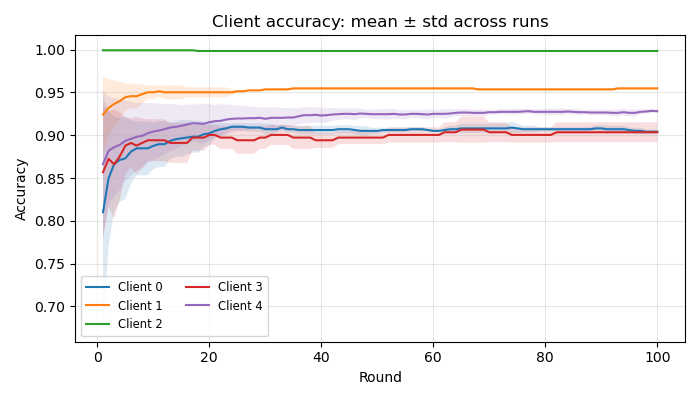}
	\caption{\textbf{Client Accuracy over Communication Rounds.} Accuracy trajectories for a random set of clients using the synthetic dataset over $100$ communication rounds. Shaded regions represent one standard deviation across multiple runs.}
	
	\label{fig:client_accuracy_trajectory}
\end{figure}

\subsection{Uplifting Clients with Poor Performance}

To demonstrate how \myfls framework mitigates negative knowledge transfer from poorly performing clients while also helping them to improve their performance, we conducted an additional experiment retaining most of the hyperparameter settings mentioned in the main manuscript. In this experiment, for clarity in presenting the results, we compare accuracy gains only against the classical FedAvg algorithm. Additionally, in contrast to the default settings mentioned earlier, we have reduced the number of clients from $C=100$ to $C=20$ and assumed a full participation i.e., $r=1$. This experiment follows a three-stage evaluation protocol. 

\begin{enumerate}
	\item \emph{Local-only baseline.}  
	Each client trained an MLP \emph{exclusively} on its own data $E=15$ epochs. The resulting test accuracy is denoted
	$\text{Acc}^{\text{local}}_c$.
	\item \emph{FedAvg baseline.}  
	We reset all models to the common global initialisation and executed
	vanilla FedAvg for the same number of communication rounds used by
	\myfls.  The final accuracy is denoted as $\text{Acc}^{\text{FA}}_c$.
	\item \emph{\myfls}  
	Starting again from the common initialisation, we ran our full
	similarity–aware, dual-anchored algorithm and obtained
	accuracy $\text{Acc}^{\text{SAPE}}_c$.
\end{enumerate}

For every client $c$ we then compute the \emph{accuracy gains}:
\[\Delta^{\text{FA}}_c = \text{Acc}^{\text{FA}}_c - \text{Acc}^{\text{local}}_c  \]
\[\Delta^{\text{SAPE}}_c = \text{Acc}^{\text{SAPE}}_c -
\text{Acc}^{\text{local}}_c\]

Next we select the bottom $25\%$ of clients (poor-performing) according to $\text{Acc}^{\text{local}}_c$ and show the margin by which these clients get improved via \myfls framework. In Figure \ref{fig:accuracy_gains} we plots these accuracy gains. Figure \ref{fig:accuracy_gains} depicts the accuracy gains in one instance of the the experiment where clients $C1, C5, C4, C13, C10$ are the $5$ clients (25\% out of all 20 clients) that are most affected by data heterogeneity i.e., clients with poor predictive models. From the plot it is evident that the margin of gain using \myfls (orange) framework outperforms classicial FedAvg (blue) algorithm. Therefore these empirical results confirm that the similarity-weighted global update together with the peer-anchored regularizer delivers pronounced benefits in avoiding negative knowledge transfer. 

\begin{figure}[t]
	\centering
	\includegraphics[width=0.8\columnwidth]{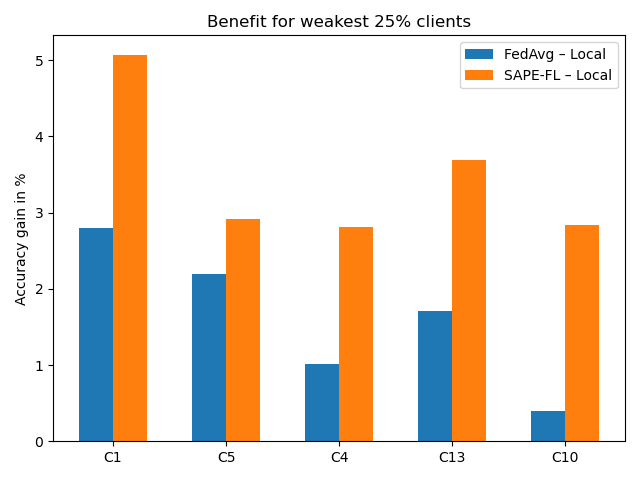}
	\caption{\textbf{Accuracy Gain of Weak Clients.} Comparison of local accuracy improvements achieved by \myfls and FedAvg for the lowest-performing $25\%$ of total clients ($C=20$). The bar chart shows the per-client accuracy gain for those $5$ selected clients.}
	\label{fig:accuracy_gains}
\end{figure}

\subsection{Influence Factor $\delta$ and Similarity Scores $S$}

To investigate the relative contribution of output similarity and weight similarity in our proposed similarity measure mentioned in Eq. (\ref{eq:def_reg_strength}), we conducted an ablation study by varying the influence factor $\delta \in \{0, 0.25, 0.5, 0.75, 1\}$. These values determine the extent to which the final similarity score depends on output-based versus weight-based alignment, with $\delta = 1$ using only output similarity and $\delta = 0$ relying solely on weight similarity. The experimental settings were kept identical to those used in our main experiments as mentioned in the main manuscript. The experimental results are mentioned in Table \ref{tab:ablation_delta}. Empirically, we observed that higher values of $\delta$ consistently led to improved performance across datasets. This aligns with our hypothesis that output similarity is a more reliable indicator of task alignment, as it reflects functional behavior even when client models converge to different but equally good local optima in heterogeneous settings. In contrast, relying solely on weight similarity still provides a useful approximation but can be relatively less effective due to the potential misalignment of model parameters that does not necessarily reflect true functional similarity. 

Additionally, we intentionally use the same tunable influence factor to maintain consistency
and stability in the relative weighting of output-space and weight-space similarities across anchors. This shared $\delta$ ensures that both peer and global similarities are measured on a comparable scale, avoiding conflicting gradient directions that can arise from independently tuned factors. Our empirical results show that a single factor suffices to balance structural and functional alignment for
both anchors. While distinct $\delta$ values for peer and global anchors could be explored in future work, the current unified formulation already provides stable results across datasets without additional tuning burden.

\begin{table*}[h]
	\centering
	\caption{\textbf{Ablation Study on Influence Factor $\delta$}. Mean test accuracies of the \myfls framework under varying contributions of output similarity and weight similarity in the computation of the similarity score. Each cell represents the mean accuracy across all clients at the end of $200$ communication rounds. The best-performing result in each row is highlighted in bold.}\label{tab:ablation_delta}
	\begin{tabular}{|c|c|c|c|c|c|}
		\toprule
		\textbf{Dataset} & \textbf{$\delta=0$} & \textbf{$\delta=0.25$} & \textbf{$\delta=0.5$}&\textbf{$\delta = 0.75$} & \textbf{$\delta=1$} \\
		\midrule
		\textbf{Synthetic} &$86.30\pm0.66$&$86.72\pm0.61$&$87.49\pm0.75$&$\mathbf{89.12\pm0.54}$&$85.07\pm0.27$ \\
		\textbf{HAR} &$78.14\pm0.26$&$79.06\pm0.41$&$81.44\pm0.78$&$\mathbf{81.74\pm0.63}$&$77.32\pm0.88$ \\
		\textbf{MNIST} &$80.05\pm0.12$&$80.82\pm0.09$&$\mathbf{81.02\pm0.33}$& $80.35\pm0.24$&$78.77\pm0.69$\\
		\textbf{EMNIST} &$73.29\pm0.73$ & $73.06\pm0.22$&$74.17\pm0.68$& $\mathbf{75.39\pm0.15}$ &$72.94\pm0.66$ \\
		\bottomrule
	\end{tabular}
\end{table*}

\subsection{Impact of Data Heterogeneity $\kappa$}

To evaluate the robustness of \myfls under varying degrees of statistical heterogeneity, we conducted additional experiments by controlling the Dirichlet concentration parameter $\kappa \in \{0.2, 0.3, 0.4, 0.5, 0.6, 1.0\}$, where smaller values of $\kappa$ correspond to more severe data skew across clients. As shown in Table~\ref{tab:heterogeneity_results}, the performance is strongest at $\kappa = 0.3$ across all the considered datasets (Synthetic: $90.66\%$, HAR: $81.22\%$, AccGyro: $81.24\%$, MNIST: $80.98\%$), highlighting the effectiveness of the proposed similarity-aware personalization and aggregation mechanisms under non-IID conditions. As $\kappa$ increases and client distributions becomes less skewed, the performance decreases gradually (e.g., MNIST drops from $80.98\%$ at $\kappa=0.3$ to $78.82\%$ at $\kappa=0.6$), which is consistent with the reduced benefit of personalization when heterogeneity is weaker. Nevertheless, \myfls remains stable and competitive across all considered $\kappa$ values, confirming that it maintains robustness across varying levels of statistical heterogeneity.

\begin{table*}[h]
	\centering
	\caption{Performance (in accuracy) of \myfls under different levels of data heterogeneity controlled by the Dirichlet concentration parameter $\kappa$. Smaller $\kappa$ indicates more non-IID data distributions across clients. Accuracy scores in bold indicate the best and higher the better.}
	\label{tab:heterogeneity_results}
	\begin{tabular}{|c|c|c|c|c|c|c|}
		\toprule
		\textbf{Datasets} & $\mathbf{\kappa=0.2}$ & $\mathbf{\kappa=0.3}$ & $\mathbf{\kappa=0.4}$ & $\mathbf{\kappa=0.5}$ & $\mathbf{\kappa=0.6}$ & $\mathbf{\kappa=1}$\\
		\midrule
		\textbf{Synthetic} & $89.24\pm0.25$ & $\mathbf{90.66\pm0.87}$ & $87.12\pm0.44$ & $86.94\pm0.71$ & $86.56\pm0.22$ & $86.01\pm0.79$ \\
		\textbf{HAR} & $79.44\pm0.31$ & $\mathbf{81.22\pm0.56}$ & $80.91\pm0.40$ & $80.05\pm0.25$ & $79.87\pm0.66$ & $76.08\pm0.53$ \\
		\textbf{AccGyro} & $80.22\pm0.39$ & $\mathbf{81.24\pm0.11}$ & $80.81\pm0.68$ & $80.13\pm0.29$ & $79.61\pm0.18$ & $79.06\pm0.22$  \\
		\textbf{MNIST} & $80.07\pm0.60$ & $80.98 \pm0.17$ & $\mathbf{81.04\pm0.19}$ & $79.23\pm0.87$ & $78.82\pm0.05$ & $78.01\pm0.20$ \\
		\bottomrule
	\end{tabular}
\end{table*}

\begin{table*}[!h]
	\centering
	\caption{\textbf{Ablation Study on Threshold $S_{\text{min}}$}. Mean test accuracies of the \myfls framework under varying similarity threshold $S_{\text{min}}$. Each cell represents the mean accuracy across all clients at the end of $100$ communication rounds. The best-performing result in each row is highlighted in bold.}\label{tab:ablation_smin}
	\begin{tabular}{|c|c|c|c|c|}
		\toprule
		\textbf{Dataset} & \textbf{$S_{\text{min}}=0.2$} & \textbf{$S_{\text{min}}=0.4$} & \textbf{$S_{\text{min}}=0.6$}&\textbf{$S_{\text{min}}=0.8$} \\
		\midrule
		\textbf{Synthetic} &$87.21\pm0.66$&$88.14\pm0.55$&$\mathbf{88.63\pm0.49}$&$86.60\pm0.42$ \\
		\textbf{HAR} &$75.21\pm 0.39$&$77.08\pm0.65$&$\mathbf{78.94\pm0.0.77}$&$77.09\pm0.51$\\
		\textbf{MNIST} &$77.29\pm0.17$&$78.01\pm0.33$&${78.55\pm0.32}$& $\mathbf{79.69\pm0.38}$\\
		
		\bottomrule
	\end{tabular}
\end{table*}

\subsection{Effect of Similarity Threshold $S_{\text{min}}$}

The similarity threshold $S_{\text{min}}$ plays a critical role in filtering out dissimilar peers during the computation of the peer-averaged anchor in our \myfls framework. This threshold governs the inclusion of clients in the peer set, thereby directly influencing the regularization term that guides local optimization. A lower value of $S_{\text{min}}$ permits broader collaboration by including a wider range of clients, but at the risk of negative transfer from poorly aligned or noisy peers. Conversely, a higher value enforces stricter trust, ensuring only highly similar clients influence each other, but potentially resulting in under-utilization of peer support, especially in highly heterogeneous or sparse client distributions. We have conducted an ablation study to understand the effects of the similarity threshold required for constructing peer-averaged model. In this experiment we use the subset of datasets described earlier. We conducted experiments with a range of values $S_{\text{min}} \in \{0.2, 0.4, 0.6, 0.8\}$. As shown in Table~\ref{tab:ablation_smin}, we found that thresholds in the range of $[0.6, 0.8]$ consistently yielded better performance across datasets. This range effectively balances the trade-off between peer set size and quality, ensuring that each client aggregates from sufficiently aligned peers while still benefiting from diversity in collaboration. Lower thresholds (e.g., $0.2$ or $0.4$) tended to introduce noisy or weakly aligned peers, reducing the effectiveness of the regularization and leading to sub-optimal personalization.

\subsection{Limitations and Future Work}
While \myfls framework provides a structured and effective approach to personalization through similarity-aware regularization and aggregation, certain practical considerations arise in real-world deployments. First, computing pairwise similarities and maintaining peer sets can introduce an additional computational and communication burden, especially in large-scale deployments. However, this can be mitigated using suitable matrix approximation techniques such as low-rank factorization that reduce the computational overhead on each client. Next, although the framework is inherently designed to handle client heterogeneity, its performance relies on the presence of at least a few peers with meaningful alignment. In scenarios with highly disjoint or adversarial client distributions, we can possible think of adaptive peer filtering mechanisms or fallback strategies (defaulting to global model regularization) that can maintain personalization quality and stability.

\myfls framework demonstrated promising improvements over existing personalized federated learning approaches. However, there are several research directions open for future exploration. First, the current gradient similarity-based trust mechanism could be enhanced with a more nuanced similarity metric such as a Bregman divergence. Second, extending \myfls framework to support dynamic client participation, where clients may join or leave during training phase, could improve its applicability in real-world asynchronous FL deployments. Furthermore, integrating privacy-preserving techniques such as differential privacy or secure multi-party computation into the trust score computation remains an open challenge.

\end{document}